\documentclass[sigconf]{acmart}

\usepackage{subfigure}
\usepackage{multirow}
\usepackage{pifont}
\usepackage{bbding}
\usepackage{adjustbox}
\usepackage{threeparttable}
\usepackage{colortbl}

\usepackage{amsthm}

\newtheorem{theorem}{Theorem}

\definecolor{COLOR_MEAN}{HTML}{f0f0f0}

\usepackage{float}
\usepackage{caption}
\usepackage{tcolorbox}
\tcbuselibrary{skins, breakable}

\newfloat{prompt}{htbp}{lop}
\floatname{prompt}{Prompt}

\newcommand{\vpara}[1]{\vspace{0.05in}\noindent\textbf{#1 }}

\AtBeginDocument{%
  }

\copyrightyear{2026}
\acmYear{2026}
\setcopyright{cc}
\setcctype{by-nc-nd}
\acmConference[KDD '26]{Proceedings of the 32nd ACM SIGKDD Conference on Knowledge Discovery and Data Mining V.1}{August 09--13, 2026}{Jeju Island, Republic of Korea}
\acmBooktitle{Proceedings of the 32nd ACM SIGKDD Conference on Knowledge Discovery and Data Mining V.1 (KDD '26), August 09--13, 2026, Jeju Island, Republic of Korea}
\acmPrice{}
\acmDOI{10.1145/3770854.3780260}
\acmISBN{979-8-4007-2258-5/2026/08}

\begin{document}

\title{Learning Hierarchical Knowledge in Text-Rich Networks with Taxonomy-Informed Representation Learning}

\author{Yunhui Liu}
\orcid{0009-0006-3337-0886}
\affiliation{
\institution{State Key Laboratory for Novel Software Technology\\ Nanjing University}
\city{Nanjing}
\country{China}}
\affiliation{
\institution{Ant Group}
\city{Hangzhou}
\country{China}}
\email{lyhcloudy1225@gmail.com}

\author{Yongchao Liu}
\authornote{Corresponding authors.}
\orcid{0000-0003-3440-9675}
\affiliation{
\institution{Ant Group}
\city{Hangzhou}
\country{China}}
\email{yongchao.ly@antgroup.com}

\author{Yinfeng Chen}
\orcid{0009-0005-6873-8786}
\affiliation{
\institution{State Key Laboratory for Novel Software Technology\\ Nanjing University}
\city{Nanjing}
\country{China}}
\email{yinfengchen@smail.nju.edu.cn}

\author{Chuntao Hong}
\orcid{0009-0009-3472-6102}
\affiliation{
\institution{Ant Group}
\city{Beijing}
\country{China}}
\email{chuntao.hct@antgroup.com}

\author{Tao Zheng}
\orcid{0009-0001-3736-4604}
\affiliation{
\institution{State Key Laboratory for Novel Software Technology\\ Nanjing University}
\city{Nanjing}
\country{China}}
\email{zt@nju.edu.cn}

\author{Tieke He}
\authornotemark[1]
\orcid{0000-0001-9649-1796}
\affiliation{
\institution{State Key Laboratory for Novel Software Technology\\ Nanjing University}
\city{Nanjing}
\country{China}}
\email{hetieke@gmail.com}

\renewcommand{\shortauthors}{Yunhui Liu et al.}

\begin{abstract}
Hierarchical knowledge structures are ubiquitous across real-world domains and play a vital role in organizing information from coarse to fine semantic levels. 
While such structures have been widely used in taxonomy systems, biomedical ontologies, and retrieval-augmented generation, their potential remains underexplored in the context of Text-Rich Networks (TRNs), where each node contains rich textual content and edges encode semantic relationships. 
Existing methods for learning on TRNs often focus on flat semantic modeling, overlooking the inherent hierarchical semantics embedded in textual documents.
To this end, we propose TIER (Hierarchical \textbf{T}axonomy-\textbf{I}nformed R\textbf{E}presentation Learning on Text-\textbf{R}ich Networks), which first constructs an implicit hierarchical taxonomy and then integrates it into the learned node representations. 
Specifically, TIER employs similarity-guided contrastive learning to build a clustering-friendly embedding space, upon which it performs hierarchical K-Means followed by LLM-powered clustering refinement to enable semantically coherent taxonomy construction. 
Leveraging the resulting taxonomy, TIER introduces a cophenetic correlation coefficient-based regularization loss to align the learned embeddings with the hierarchical structure.
By learning representations that respect both fine-grained and coarse-grained semantics, TIER enables more interpretable and structured modeling of real-world TRNs. We demonstrate that our approach significantly outperforms existing methods on multiple datasets across diverse domains, highlighting the importance of hierarchical knowledge learning for TRNs.
Our code is open-sourced at \url{https://github.com/Cloudy1225/TIER}.
\end{abstract}




\maketitle

\section{Introduction}

Hierarchical knowledge plays a fundamental role across diverse domains, enabling systems to represent, reason about, and retrieve information at multiple levels of abstraction. 
From taxonomy-guided document classification~\cite{HQC, TELEClass} and biomedical ontologies~\cite{BioOnto, BioOnto1} to modern retrieval-augmented generation systems~\cite{GraphRAGSurvey, HiRAG} and LLM-based question-answering pipelines~\cite{HamQA, OBQC}, hierarchies serve as powerful inductive priors that provide a coarse-to-fine lens through which complex information can be structured, interpreted, and generalized, ultimately improving both performance and explainability of learning systems.

Despite its broad utility, hierarchical modeling remains underexplored in the context of Text-Rich Networks (TRNs), where each node is associated with rich textual content and connected through semantically meaningful edges~\cite{LLM4Graph}. 
For example, academic papers connected via citations often discuss related topics, and e-commerce products linked by co-purchase behaviors usually serve similar functions. 
Together, TRNs leverage the dual strengths of textual semantics and graph structure, and have therefore been widely adopted in numerous applications, including recommender systems~\cite{LLMRec}, anomalous user detection~\cite{LGB}, social media analysis~\cite{TwHIN-BERT}, and scientific document understanding~\cite{SciLLM}. 
Effective TRN representation learning requires integrating both the textual content and the graph structure~\cite{CS-TAG, HTAG}. 
While early approaches used shallow text features and focused on Graph Neural Network (GNN) design~\cite{GCN, SAGE, GAT}, recent models have leveraged pre-trained language models (PLMs)~\cite{BERT, RoBERTa, SentenceBERT} and more recently LLMs~\cite{Mistral7B, GPT4, DeepSeekV3} to incorporate richer semantic information. 
These hybrid models have demonstrated strong performance improvements through strategies such as variational inference~\cite{GLEM}, masked prediction~\cite{PATTON}, contrastive learning~\cite{GRENADE}, or direct LLM inference~\cite{OFA, GraphWiz}. 

\begin{figure}[h]
\centerline{\includegraphics[width=1.\linewidth]{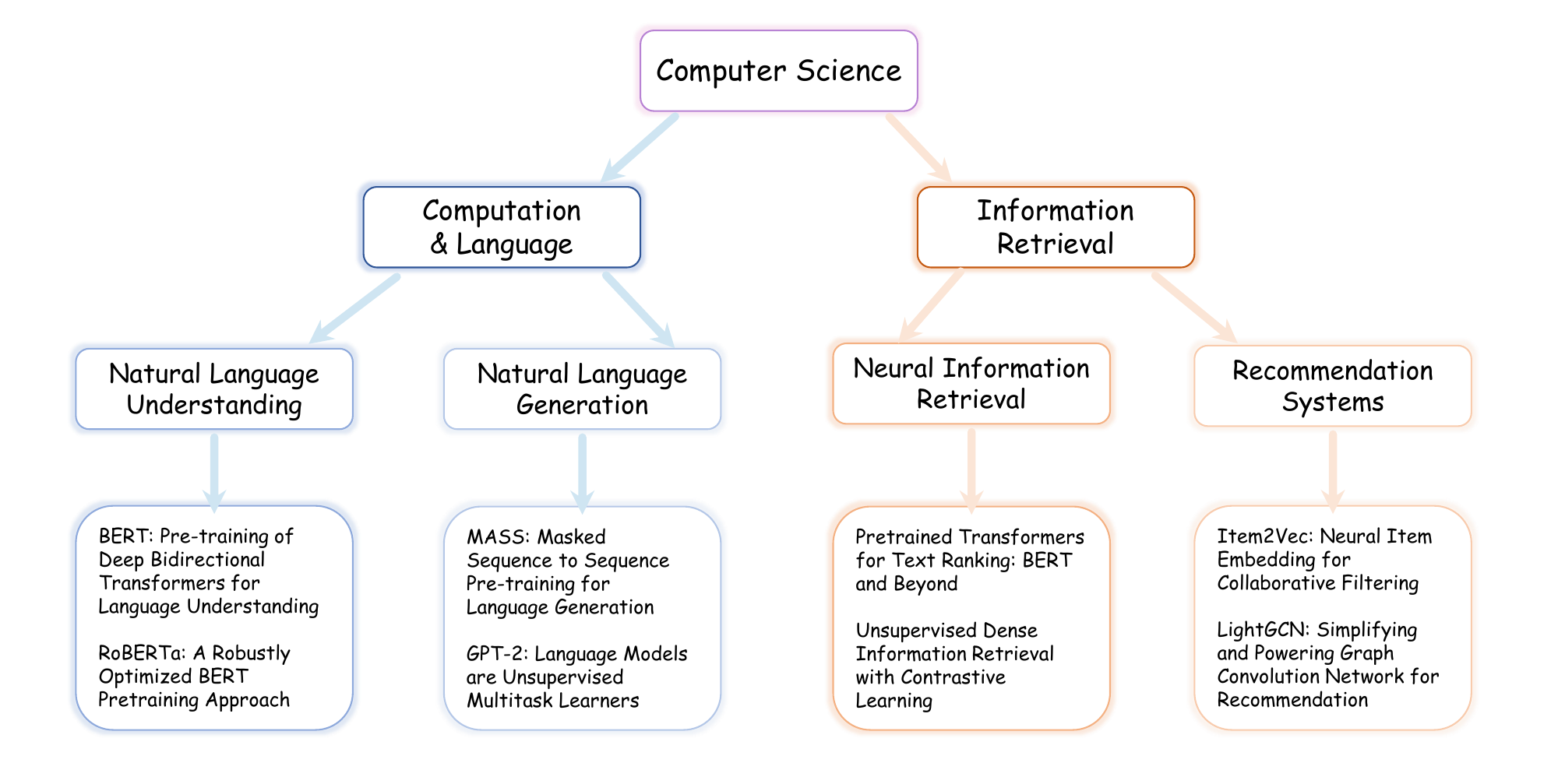}}\caption{An example taxonomy of computer science papers.} \label{fig:taxonomy_example}
\end{figure}

However, most existing TRN learning methods concentrate on preserving flat semantic structures during their modeling process, but rarely consider and leverage the hierarchical semantic structures that exist ubiquitously in real-world documents. 
In many real-world domains, the associated textual content follows a tree-structured hierarchical taxonomy (e.g., 
CCS\footnote{ACM Computing Classification System: \url{https://dl.acm.org/ccs}}, 
MeSH\footnote{Medical Subject Headings: \url{https://www.nlm.nih.gov/mesh/meshhome.html}}, 
APC\footnote{Amazon Product Categories: \url{https://www.ecomengine.com/amazon-categories}}, 
and 
IPC\footnote{International Patent Classification: \url{https://www.wipo.int/en/web/classification-ipc}}
) that encapsulates coarse-to-fine semantic information~\cite{TaxoSurvey, TaxoGen, NetTaxo}. 
As illustrated in Figure~\ref{fig:taxonomy_example}, a group of papers related to "Computation \& Language" and "Information Retrieval" can be further decomposed into finer-grained categories like "Natural Language Understanding", "Natural Language Generation", "Neural Information Retrieval", and "Recommender Systems", or conversely abstracted into coarser-grained categories like "Computer Science". 
An ideal representation space is expected to reflect this inherent hierarchical semantic structure. 
In such a space, a paper from "Natural Language Understanding" should be most similar to other papers in the same category, followed by papers in "Natural Language Generation" from the same coarse-grained category, and finally, less similar papers from other coarse-grained categories, like "Information Retrieval".
This hierarchical structure allows the model to distinguish between documents based on their semantic closeness not just at a surface level, but across different levels of abstraction.
By explicitly incorporating hierarchical semantic structures, node representations can more accurately reflect the intrinsic relationships between documents, thus improving the model's ability to classify and cluster documents effectively. 
However, effectively modeling this hierarchy in TRN learning has not been thoroughly explored. Addressing this gap requires careful consideration of how to model and integrate hierarchical semantic structures in TRN learning models. 
Thus, the core challenges addressed in this paper are: (1) How to effectively model the hierarchical structure of semantics inherent in TRNs when explicit hierarchies are absent or incomplete? (2) How to seamlessly integrate this hierarchical structure into TRN learning models, ensuring that the learned node representations reflect both fine-grained and coarse-grained semantic relations?

To address these challenges, we propose TIER (Hierarchical \textbf{T}axonomy-\textbf{I}nformed R\textbf{E}presentation Learning on Text-\textbf{R}ich Networks), which explicitly models and leverages hierarchical semantics in TRNs. 
TIER operates in two key stages. 
First, it automatically constructs a high-quality, semantically coherent taxonomy that captures the coarse-to-fine structure of node semantics. 
This is achieved by embedding both textual semantics and graph topology into a clustering-friendly representation space through a similarity-guided contrastive learning objective. 
Based on this space, we apply a bottom-up hierarchical K-Means algorithm to induce a multi-level taxonomy, and further refine it using LLM-guided clustering operations to improve semantic coherence, interpretability, and robustness. 
Second, to ensure the learned node embeddings reflect the constructed taxonomy, TIER introduces a taxonomy-informed regularization mechanism based on the Cophenetic Correlation Coefficient. 
This regularizer aligns the structure of the learned embedding space with the semantic distances encoded in the taxonomy, encouraging node representations to reflect both fine-grained and coarse-grained semantic relations.
By jointly learning node representations and hierarchical taxonomies, TIER enables more faithful modeling of real-world document semantics and provides hierarchically structured inductive signals that benefit downstream tasks such as node classification. 
Extensive experiments across multiple datasets from diverse domains demonstrate the effectiveness of hierarchical learning in our approach. 

\section{Related Work}

\subsection{Hierarchical Knowledge Learning}
Hierarchical knowledge learning aims to encode multi-level structures, such as topic taxonomies, label hierarchies, or class ontologies, into embeddings that preserve the underlying tree-like semantics.
This paradigm has demonstrated broad applicability across various domains, such as text classification~\cite{TELEClass}, image classification~\cite{HCPCC}, document hashing~\cite{HierHash}, network embedding~\cite{NetHiex}, and retrieval-augmented generation~\cite{HiRAG}. 
Specifically, one common line of work adopts contrastive learning~\cite{HGCLR, HiMulConE, HCSC} to model hierarchical relations by pulling together samples from the same sub-tree and pushing apart those from different branches. 
Alternatively, hyperbolic learning methods~\cite{HIE, HypStructure, HiT} exploit the geometry of non-Euclidean spaces to capture hierarchical distance and containment more naturally and efficiently. 
In the graph domain, recent efforts~\cite{Taxo-GNN, TG-GNN} also explored integrating external hierarchical taxonomies with GNNs to enhance node representations, further highlighting the utility of structured semantic priors in representation learning. 
However, most existing methods rely heavily on supervised signals from carefully annotated hierarchical datasets, limiting their applicability in real-world scenarios where such taxonomies are unavailable or incomplete. 
In this work, we address this limitation by first constructing a high-quality hierarchical taxonomy via hierarchical contrastive clustering and LLM-powered clustering refinement. Building upon this induced hierarchy, we further introduce a Cophenetic Correlation Coefficient-based regularization to explicitly align the learned representation space with the underlying hierarchical structure.

\subsection{Text-Rich Network Representation Learning}
At the core of learning representations on TRNs lies effectively integrating textual semantics with structural connections to produce informative node representations~\cite{CS-TAG, HTAG}. 
Early methods primarily relied on shallow textual features such as bag-of-words and skip-gram embeddings~\cite{SkipGram}, while emphasizing the design of advanced GNN architectures~\cite{GCN, SAGE, GAT}. 
Later, the advent of pre-trained language models (PLMs)~\cite{BERT, RoBERTa, SentenceBERT} brought significant improvements in capturing the contextual and nuanced meaning of text. 
These approaches integrate PLMs with GNNs or Graph Transformers~\cite{GraphFormers} using techniques such as variational inference~\cite{GLEM}, contrastive learning~\cite{GRENADE}, and masked node prediction~\cite{PATTON}, leading to significant improvements in TRN representation learning.
More recently, the emergence of large language models (LLMs)~\cite{Mistral7B, GPT4, DeepSeekV3} has introduced a new paradigm for TRN learning, due to their powerful context awareness and semantic reasoning capabilities.  
Specifically, LLMs as Enhancers can retrieve external semantic knowledge pertinent to each node to refine GNN initialization~\cite{TAPE, GAugLLM}. 
LLMs as Predictors are able to encode graph structures via carefully designed prompts and directly make predictions in an autoregressive manner~\cite{OFA, GraphWiz}. 
Moreover, LLMs as Aligners enable the alignment of LLM and GNN representations in a shared embedding space, enriching GNNs with the semantic knowledge embedded in LLMs~\cite{GraphAdapter, GraphCLIP}. 
For a comprehensive overview, readers are referred to recent surveys~\cite{LLM4Graph, LLM4Graph2, LLM4Graph3, GFM} and benchmarks~\cite{CS-TAG, GLBench, TSGFM, LLMNodeBed} on TRN learning.

\section{Preliminary}
\subsection{Problem Statement}
Our goal is to learn node representations that capture the hierarchical semantics embedded in TRNs.
Formally, a TRN is defined as $\mathcal{G} = (\mathcal{V}, \mathcal{E}, \mathcal{S})$, where $\mathcal{V}$ is the set of nodes, $\mathcal{E}$ the set of edges, and $\mathcal{S} = \{s_1, \ldots, s_n\}$ the collection of textual document associated with each node. The adjacency matrix is denoted as $\boldsymbol{A} \in \{0,1\}^{n \times n}$. 
To assess the quality and utility of learned representations, we evaluate them on node classification as a representative downstream task, where labels are available for a subset of nodes $\mathcal{V}_l$ and the goal is to predict the labels of the remaining nodes $\mathcal{V}_u$.

Traditionally, the textual attributes of nodes are encoded into $f$-dim shallow feature vectors $\boldsymbol{X} = [\boldsymbol{x}_1, \ldots, \boldsymbol{x}_n] \in \mathbb{R}^{n \times f}$ using naïve methods such as bag-of-words. 
Such transformation is adopted in most GNN papers. 
In contrast, recent LLM-based approaches directly consume the raw textual input, leveraging the rich contextual knowledge embedded in pre-trained LMs. Such models offer the potential to generate semantically rich and task-relevant embeddings, potentially improving node classification performance.

\subsection{Hierarchical Taxonomy}
A taxonomy is a specialized form of hierarchical knowledge graph that captures hypernym-hyponym (i.e., "is-a") relationships between concepts or entities~\cite{TaxoSurvey, TaxoGen}. 
Formally, a taxonomy can be represented as a tree $\mathcal{T} = (\mathcal{U}, \mathcal{R})$, where $\mathcal{U} = \{ u_1, u_2, \dots, u_{|\mathcal{U}|} \}$ denotes the set of all taxonomy categories (nodes) within the hierarchy and $\mathcal{R} = \{ r_1, r_2, \dots, r_{|\mathcal{R}|} \}$ denotes the set of directed edges, where each edge $r_i = (u_p, u_c)$ connects a parent node $u_p$ (the hypernym) to its child node $u_c$ (the hyponym).

Intuitively, the tree $\mathcal{T}$ defines a hierarchy of semantic categories that organizes domain knowledge in a structured and interpretable manner. 
For example, as illustrated in Figure~\ref{fig:taxonomy_example}, the root node "Computer Science" branches into subcategories such as "Computation \& Language" and "Information Retrieval." 
Consequently, semantically related topics like "Natural Language Understanding" and "Natural Language Generation" appear closer in the hierarchy, while more distant topics, e.g., "Natural Language Understanding" and "Recommender Systems", are placed further apart, reflecting their lower semantic proximity.

\begin{figure*}[h]
\centerline{\includegraphics[width=1.\linewidth]{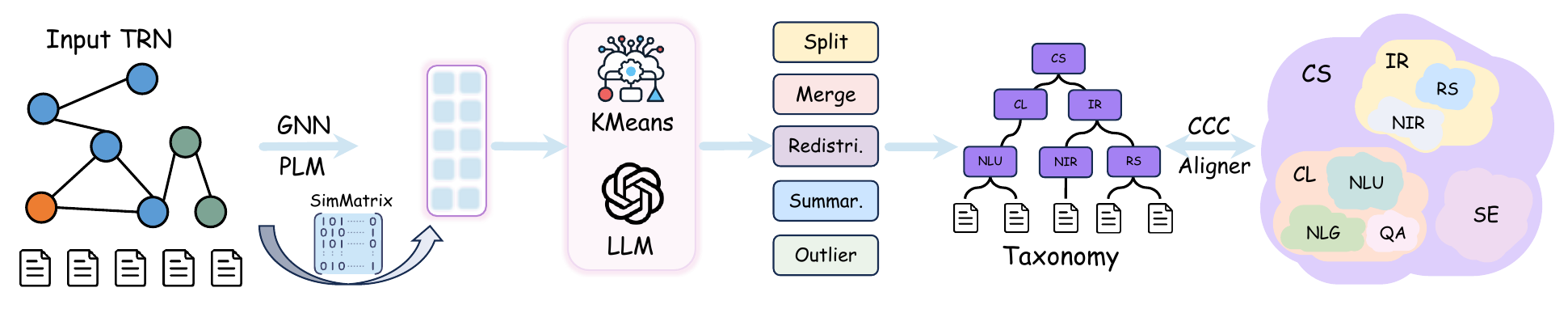}}\caption{The framework of TIER.} \label{fig:framework}
\end{figure*}

\section{Methodology}
In this section, we present our TIER framework (Figure~\ref{fig:framework}), which first constructs an implicit hierarchical taxonomy and then integrates it into node representation learning.

\subsection{Hierarchical Taxonomy Construction}
As previously discussed, taxonomies inherently encode gradual and structured semantic information among real-world documents, and effectively capturing them in node representations can significantly enhance semantic understanding in TRNs.
However, existing TRN datasets lack ground-truth taxonomies and corresponding node-level assignments, which may be attributed to the high cost of manual annotation. 
This calls for the automatic construction of high-quality and semantically coherent hierarchical taxonomies.
Existing taxonomy construction methods, however, are typically designed for individual documents~\cite{TaxoSurvey, TaxoGen} and largely neglect the graph topology that is central to TRNs.
For instance, in citation networks, citation links reflect the relevance and influence between papers, and papers with stronger connections are typically more semantically aligned. 
Motivated by this, we propose a taxonomy construction framework that synergistically integrates textual content and graph topology to uncover implicit hierarchical structures in TRNs. 
We first employ similarity-guided contrastive learning to jointly encode textual and structural information into a clustering-friendly and task-aware representation space. Based on this space, we then apply a bottom-up hierarchical K-Means algorithm, augmented with LLM-powered clustering refinement, to construct a high-quality and semantically coherent taxonomy tree. 
The details are elaborated as follows.

\subsubsection{Similarity-Guided Contrastive Learning}
The first step of our framework aims to learn clustering-friendly node representations that pull semantically similar nodes closer in the embedding space. 
Recent theoretical and empirical studies have shown that contrastive learning can produce representations that align well with semantic classes~\cite{SCL, NS4GC}, and even with hierarchical class structures~\cite{RESSL}. 
Motivated by these findings, we design a contrastive pretraining stage, guided by pairwise semantic similarities derived from both available labels and the underlying graph structure. 
Specifically, we build upon standard graph contrastive learning pipelines~\cite{GRACE, BGRL, SSGE}, which begin by applying random edge dropping and feature masking to the input graph $\mathcal{G} = (\boldsymbol{A}, \boldsymbol{X})$, to generate two augmented views, $\mathcal{G}^1 = (\boldsymbol{A}^1, \boldsymbol{X}^1)$ and $\mathcal{G}^2 = (\boldsymbol{A}^2, \boldsymbol{X}^2)$. 
These views are then encoded using a shared GNN encoder, producing $\ell_2$-normalized node embeddings $\boldsymbol{Z}_1, \boldsymbol{Z}_2 \in \mathbb{R}^{n \times d}$. 
We introduce a Similarity-Guided Contrastive Loss defined as:
\begin{equation}\label{eq:sgcl}
    \mathcal{L}_{scl} = - \sum_{\boldsymbol{S}_{ij}=1} {\boldsymbol{z}_1^i}^\top \boldsymbol{z}_2^j + \gamma \sum_{\boldsymbol{S}_{ij}=0}({\boldsymbol{z}_1^i}^\top \boldsymbol{z}_2^j)^2.
\end{equation}
Here, $\boldsymbol{S} \in \{0,1\}^{n \times n}$ is a semantic similarity indicator matrix, where $\boldsymbol{S}_{ij} = 1$ indicates that nodes $i$ and $j$ are semantically similar (e.g., belong to the same class) and should be treated as a positive pair; otherwise, they are considered a negative pair. Similar to classic contrastive losses~\cite{SimCLR, SCL}, this objective pulls together positive pairs while pushing apart negative ones, with a trade-off coefficient $\gamma$. 

Ideally, the optimal similarity matrix would be $\boldsymbol{S}^* = \boldsymbol{Y} \boldsymbol{Y}^\top$, where $\boldsymbol{Y}$ is the one-hot label matrix of all nodes. However, this ideal is unattainable in practice due to the limited availability of labeled nodes.
Classic contrastive learning methods~\cite{SimCLR, SCL, GRACE} inherently set $\boldsymbol{S} = \boldsymbol{I}$ (i.e., identity matrix), treating each sample as only positive with respect to itself while repelling all other nodes. Such a naïve identity-based approach inevitably hampers the clustering performance by discarding useful semantic signals. 
To mitigate this, we construct a more informative similarity matrix $\boldsymbol{S}$ by expanding positive pairs from two sources:

\begin{itemize}
    \item \textbf{Label-based Similarity}: If both nodes $i$ and $j$ are labeled and share the same class, we set $\boldsymbol{S}_{ij} = 1$. This is analogous to Supervised Contrastive Learning (SupCon)~\cite{SupCon}, where embeddings of intra-class samples are pulled together.

    \item \textbf{Structure-based Similarity}: If at least one of $i$ or $j$ is unlabeled, but the two nodes are directly connected in the graph, we also set $\boldsymbol{S}_{ij} = 1$. This is based on the homophily assumption (i.e., connected nodes are likely to be semantically similar), which is widely observed in real-world graphs and has been effectively leveraged in prior works~\cite{GCN, HomoGCL, BLNN}.
\end{itemize}

\noindent The following theoretical result further supports our design:

\begin{theorem}\label{the}
    Given a graph with edge homophily $h > 0.5$, the constructed similarity matrix $\boldsymbol{S}$ more closely approximates the ideal matrix $\boldsymbol{S}^*$ compared to both classic contrastive learning~\cite{SimCLR, SCL} and SupCon~\cite{SupCon}. 
    Consequently, minimizing the similarity-guided contrastive loss Eq.~\eqref{eq:sgcl} with $\boldsymbol{S}$ leads to node representations that better preserve the semantic cluster structure.
\end{theorem}

The full proof is provided in Appendix~\ref{app:proof}. Through Similarity-Guided Contrastive Learning, we effectively embed both textual content and topological structure into a clustering-friendly and task-aware representation space, facilitating the subsequent construction of a high-quality hierarchical taxonomy.

\subsubsection{LLM-Powered Hierarchical Clustering}\label{sec:llm_hierarchical_clustering}
Given the node representations obtained from the previous stage, we aim to construct a semantically coherent hierarchical taxonomy tree $\mathcal{T}$. 
To achieve this, we adopt a bottom-up hierarchical K-Means to organize node representations into an $L$-level taxonomy, where each layer $l$ contains $k_l$ semantic clusters encapsulating coarse-to-fine semantics. 
The detailed procedure is as follows:

\vpara{Finest-Grained Cluster Initialization.}
We first extract node representations $\boldsymbol{Z} \in \mathbb{R}^{n \times d}$ using the pretrained GNN encoder. 
These representations are then clustered using K-Means into $k_L$ clusters, denoted as $\{ \mathcal{C}_i \}_{i=1}^{k_L}$, where $k_L$ is the number of clusters at the finest hierarchy. The resulting cluster centers form a matrix $\boldsymbol{C} = [\boldsymbol{c}_1, \dots, \boldsymbol{c}_{k_L}] \in \mathbb{R}^{k_L \times d}$, which serves as the prototype embeddings for this layer. This finest-level clustering serves as the foundation for constructing higher levels of the taxonomy.

\vpara{LLM-Assisted Clustering Refinement.}
The semantic quality of the finest-level clusters is critical, as it forms the basis for the entire hierarchy. 
However, standard K-Means often yields clusters that are geometrically coherent but semantically misaligned~\cite{FracDist}, especially when semantically distinct documents are close in the embedding space. 
To overcome this limitation, we further introduce an LLM-guided refinement pipeline to enhance cluster cohesion, remove noise, and adjust incorrect groupings. 
This refinement process integrates the strong semantic understanding capabilities of LLMs and consists of the following key steps:

\begin{enumerate}
    \item \emph{Splitting Low-Cohesion Clusters}: 
    For each cluster $\mathcal{C}_i$, we compute the average cosine similarity between each member $\boldsymbol{z} \in \mathcal{C}_i$ and the cluster centroid $\boldsymbol{c}_i$ to assess semantic compactness: 
    $
    \text{cohesion}(\mathcal{C}_i) = \frac{1}{|\mathcal{C}_i|} \sum_{\boldsymbol{z} \in \mathcal{C}_i} \cos(\boldsymbol{z}, \boldsymbol{c}_i)
    $.
    A cluster with cohesion below a threshold $\tau_{\text{split}}$ is suspected of containing multiple latent topics. 
    In this case, we randomly sample $n_{\text{split}}$ texts from the cluster and prompt the LLM using Prompt~\ref{pr:split} to decide whether the cluster should be split and suggest the number of subclusters. 
    This process allows us to identify cases where geometrically close but semantically diverse documents have been incorrectly grouped together, ensuring that each resulting cluster is more semantically focused and topically consistent. 

    \item \emph{Merging Semantically Similar Clusters}: 
    For each pair of clusters $(\mathcal{C}_i, \mathcal{C}_j)$ with high centroid similarity: 
    $\cos(\boldsymbol{c}_i, \boldsymbol{c}_j) > \tau_{\text{merge}}$, 
    we prompt the LLM using $n_{\text{merge}}$ randomly selected texts from both clusters and use Prompt~\ref{pr:merge} to assess semantic overlap. If the LLM confirms the semantic similarity, the clusters are merged into a unified group. This refinement resolves semantic fragmentation, where highly related concepts are mistakenly split, leading to a more compact and cohesive taxonomy. 

    \item \emph{Redistributing Degenerate Clusters}: 
    Clusters with fewer than $n_{\text{min}}$ samples are considered unstable or noisy. 
    We reassigned each sample $\boldsymbol{z}$ in such clusters to the nearest larger cluster: $\text{assign}(\boldsymbol{z}) = \arg\max_{j \in \mathcal{J}} \cos (\boldsymbol{z},\boldsymbol{c}_j)$, where $\mathcal{J}$ is the set of sufficiently large clusters. 
    This step eliminates noisy micro-clusters that arise from outliers or sampling variance, improving cluster stability and ensuring that the taxonomy is not distorted by small, irrelevant groups.

    \item \emph{Labeling and Summarizing Clusters}: 
    For each refined cluster $\mathcal{C}_i$, we select the top-$n_{\text{close}}$ samples closest to the centroid as representatives and prompt the LLM using Prompt~\ref{pr:summarize} to generate a natural language label and summary based on the textual content: $\text{summary}(\mathcal{C}_i) = \text{LLM}(\{\text{s}_j\}_{j \in \text{Top-}n_{\text{close}}})$. 
    The resulting summaries provide an interpretable and human-readable representation of each cluster, facilitate human understanding and the subsequent outlier reassignment. 
    By encoding the semantic context of each cluster, the label and summary guide the LLM in accurately assigning outliers to the most appropriate clusters in the next step.

    \item \emph{Reassigning Outlier Samples}: 
    To further refine boundary cases, for each cluster $\mathcal{C}_i$, we identify the bottom $r$-fraction of nodes farthest from the centroid $\boldsymbol{c}_i$ as outlier points.
    For each outlier $\boldsymbol{z}_j$, we retrieve the top-$n_{\text{outlier}}$ nearest clusters and prompt the LLM with the document text $s_j$ along with the natural language labels and summaries of these candidate clusters using Prompt~\ref{pr:outlier}. 
    The LLM then selects the cluster whose semantic summary best aligns with the outlier’s content, enabling semantically informed reassignment. 
    This process can correct misallocations that arise when vector-based similarity fails to capture nuanced contextual meanings, ensuring boundary samples are placed into semantically appropriate clusters.
\end{enumerate}

\noindent Through these refinement procedures, we obtain coherent, stable, and semantically interpretable clusters, forming the basis of the lowest level of the taxonomy tree.

\vpara{Bottom-Up Hierarchy Construction.}
With the refined finest-level clusters as the leaf nodes, we construct the taxonomy tree layer by layer in a bottom-up manner, progressively capturing semantic abstraction across multiple granularities.
We begin by treating each cluster $\mathcal{C}_i^L$ at the finest level $L$ as a taxonomy node associated with its centroid $\boldsymbol{c}_i^L \in \mathbb{R}^d$. 
To form the hierarchy, we recursively apply K-Means clustering to the set of centroids at level $l$ to produce $k_{l-1}$ super-clusters $\boldsymbol{C}^{l-1} = \{ \boldsymbol{c}^{l-1}_j \}_{j=1}^{k_{l-1}}$ at level $l-1$. 
Each super-cluster centroid at level $l-1$ becomes a new node in the hierarchy, and parent-child relations are established by assigning each cluster at level $l$ to the super-cluster it was grouped into during clustering. 
This recursive clustering continues until a single root cluster is formed. 
Accordingly, the taxonomy tree $\mathcal{T} = (\mathcal{U}, \mathcal{R})$ is defined by the set of all cluster nodes and their parent-child relations across levels: 
\begin{equation}
\mathcal{U} = \left\{\mathcal{C}^l\right\}_{l=1}^L, 
\mathcal{R} = \left\{ \left(\boldsymbol{c}_i^l, \boldsymbol{c}_j^{l-1}\right) \mid \boldsymbol{c}_i^l \text{ assigned to } \boldsymbol{c}_j^{l-1} \right\}.
\end{equation}
The entire structure is thus composed by connecting fine-grained clusters upward into increasingly abstract semantic categories. 
Through this construction, our model captures not only the latent semantic similarity between nodes but also their hierarchical relationships, enabling downstream tasks to benefit from multi-level semantic context.

\subsection{Taxonomy-Informed Representation Learning}
After obtaining the taxonomy $\mathcal{T}$, our goal is to integrate this hierarchical structure into TRN learning models such that the resulting embeddings encode both fine-grained and coarse-grained semantic relations.
To achieve this, we introduce a distance-based regularization framework that encourages consistency between the taxonomy and the learned representation space. 
Central to this approach is the Cophenetic Correlation Coefficient (CCC), which we adopt as a structure-aware measure of alignment between pairwise distances in two different metric spaces: the tree-derived taxonomy space and the latent embedding space. 

\vpara{Cophenetic Correlation Coefficient.}
CCC is a widely used measure to evaluate how faithfully a hierarchical clustering tree preserves the pairwise distances among the original observations~\cite{CCC, CCC2}. 
In our context, it serves as a differentiable supervision signal to align the Euclidean geometry of node embeddings with the semantic structure encoded in the taxonomy. 
Specifically, we treat the finest-level clusters $\{ \mathcal{C}_i \}_{i=1}^{k}$ of the taxonomy as the basic semantic units. Each cluster $\mathcal{C}_i$ is represented by a prototype vector $\boldsymbol{p}_i \in \mathbb{R}^d$ (e.g., the mean of node embeddings in the cluster). 
Let $\boldsymbol{D} \in \mathbb{R}^{k \times k}$ denote the pairwise Euclidean distance matrix between prototype embeddings with $D_{ij} = \| \boldsymbol{p}_i - \boldsymbol{p}_j \|$, and 
let $\boldsymbol{D}_{\text{coph}} \in \mathbb{R}^{k \times k}$ denote the cophenetic distance matrix, where $D_{\text{coph}, ij}$ is defined as the length of the shortest path connecting the cluster $i$ and $j$ in the taxonomy tree.
Then the CCC is computed as the Pearson correlation between the upper-triangular entries of $\boldsymbol{D}$ and $\boldsymbol{D}_{\text{coph}}$:
\begin{equation}\label{eq:CCC}
    \text{CCC}(\boldsymbol{D}, \boldsymbol{D}_{\text{coph}}) = \frac{\sum_{i<j} (D_{ij} - \bar{D})(D_{\text{coph}, ij} - \bar{D}_{\text{coph}})}{\sqrt{\sum_{i<j} (D_{ij} - \bar{D})^2} \cdot \sqrt{\sum_{i<j} (D_{\text{coph}, ij} - \bar{D}_{\text{coph}})^2}}
\end{equation}
A higher CCC (close to 1) indicates that the learned representation space closely aligns with the tree structure, whereas a lower CCC suggests a mismatch between embedding geometry and semantic hierarchy.

\vpara{CCC as a Taxonomy Aligner.}
To encourage alignment between the learned cluster embeddings and the taxonomy structure, we naturally introduce a regularization term:

\begin{equation}\label{eq:regularization}
    \mathcal{L}_{\text{CCC}} = 1 - \text{CCC}(\boldsymbol{D}, \boldsymbol{D}_{\text{coph}})
\end{equation}
This loss penalizes geometric arrangements of the node representations that deviate from the tree-based semantic distances. 
Therefore, it ensures that 
1) Encouraging hierarchical smoothness: semantically close classes (e.g., sibling nodes in the taxonomy) are regularized to be closer in representation space; 
2) Maintaining coarse-to-fine separability: higher-level distinctions (e.g., across distant branches) are also respected in the latent space.

\vpara{Overall Training Objective.}
This regularization is combined with the task-specific loss (e.g., cross-entropy) over individual node representations. 
The overall objective becomes:

\begin{equation}\label{eq:loss}
\mathcal{L}_{\text{total}} = \mathcal{L}_{\text{CE}} + \lambda \cdot \mathcal{L}_{\text{CCC}}
\end{equation}
where $\lambda$ controls the influence of the taxonomy-informed regularization. 

\vpara{Time Complexity.}
Since $\mathcal{L}_{\text{CCC}}$ is defined only over finest-level clusters, it introduces minimal computational overhead while still providing strong structural inductive bias.
For a latent space of dimension $d$, each training iteration incurs at most $O(d\min(b^2, k^2))$ additional computations, where $b$ is the batch size. 
This overhead is often negligible compared to the overall computational cost of training a neural network.

\section{Experiments}
In this section, we conduct comprehensive experiments to validate the effectiveness of our proposed TIER. 
Specifically, we aim to answer the following research questions: 
\textbf{RQ1:} How does TIER compare to state-of-the-art baselines? 
\textbf{RQ2:} Does TIER learn a hierarchical representation space?
\textbf{RQ3:} How to visually understand the constructed taxonomy? 
\textbf{RQ4:} Is TIER efficient enough for TRN representation learning? 
\textbf{RQ5:} Does each component contribute positively to TIER?  

\subsection{Experimental Setup}
\vpara{Datasets.}
To provide a comprehensive evaluation of our proposed method, we adopt 8 datasets from the LLMNodeBed~\cite{LLMNodeBed} benchmark, as summarized in Table~\ref{tab:dataset_statistics}. 
These datasets span diverse domains, including academic networks, web graphs, and product co-purchase networks, and cover a wide range of scales, from thousands to hundreds of thousands of nodes. 
Specifically, Cora~\cite{TAPE}, Citeseer~\cite{LLM4Graph3}, Pubmed~\cite{TAPE}, and ArXiv~\cite{OGB} are citation networks where the task is to classify scientific papers into distinct research topics. 
WikiCS~\cite{WikiCS} is a Wikipedia-based web graph, where nodes represent entities categorized into different Wikipedia categories. 
Books, Photo, and Computer~\cite{CS-TAG} are product networks, where the goal is to classify item descriptions into various product categories. 
Further details and dataset-specific descriptions are provided in Appendix~\ref{app:datasets}.

\vpara{Baselines.}
Following the LLMNodeBed~\cite{LLMNodeBed} benchmark, we compare our method against a diverse set of baselines spanning multiple paradigms: 
Classic GNNs with Shallow Embeddings (SE): These models use classic GNN architectures with shallow text features (e.g., bag-of-words). Specifically, we include GCN$_{\text{SE}}$~\cite{GCN}, SAGE$_{\text{SE}}$~\cite{SAGE}, and GAT$_{\text{SE}}$~\cite{GAT}. 
PLM-based Methods: We fine-tune encoder-only pre-trained language models (PLMs) of different sizes, including SenBERT (66M parameters)~\cite{SentenceBERT} and RoBERTa (355M parameters)~\cite{RoBERTa}, to obtain text embeddings which are then used for classification. 
Hybrid PLM + GNN Methods: These methods integrate PLM-based textual embeddings with GNN-based structural modeling. We evaluate GIANT~\cite{GIANT} and GLEM~\cite{GLEM} as representative approaches. 
LLM-as-Encoder Methods: These models treat LLMs as node feature encoders and inject their embeddings into graph learning pipelines. We include ENGINE~\cite{ENGINE} and GCN$_{\text{LLM}}$~\cite{LLMNodeBed} in this category.
LLM-as-Explainer Methods: In this paradigm, LLMs are used to provide auxiliary semantic knowledge or explanations for node representations. We include TAPE~\cite{TAPE} and GAugLLM~\cite{GAugLLM}.
LLM-as-Predictor Methods: These models directly use LLMs to perform node classification via prompt-based or instruction-tuned generation. We include GraphGPT~\cite{GraphGPT}, LLaGA~\cite{LLaGA}, and LLM Instruction Tuning (LLM$_{\text{IT}}$)~\cite{LLMNodeBed}.

\vpara{Implementation.}
All experiments are conducted using the most up-to-date benchmark suite LLMNodeBed~\cite{LLMNodeBed} for TRNs to date. 
Our task is semi-supervised node classification, using the dataset splits and shallow input embeddings provided by LLMNodeBed.
To ensure fair comparisons, all baseline methods are implemented with consistent architectural components: GCN~\cite{GCN} as the backbone GNN, RoBERTa (355M)~\cite{RoBERTa} as the PLM, and Mistral-7B~\cite{Mistral7B} as the default LLM where applicable.
Each experiment is repeated across 5 random seeds, and we report the mean accuracy and standard deviation.
For our proposed method, initial node embeddings are obtained using frozen RoBERTa (355M) on raw text. 
The taxonomy construction and downstream node classification are trained with separate GCN encoders. 
To refine the hierarchical clusters, we prompt the DeepSeek-V3~\cite{DeepSeekV3} via API. 
Following LLMNodeBed, we search GCN hyperparameters over: number of layers $\in \{2, 3, 4\}$, hidden dimensions $\in \{64, 128, 256\}$, dropout rates $\in \{0.3, 0.5, 0.7\}$, and consider configurations with and without batch normalization and residual connections. 
The weight $\lambda$ of the taxonomy-informed regularization loss is set to 1.0 by default. 
We run our method on a Linux server equipped with an Intel(R) Xeon(R) Gold 6248 CPU @ 2.50GHz and a 32GB NVIDIA Tesla V100 GPU. 
More implementation details, including the shapes of the constructed taxonomy trees, are presented in Appendix~\ref{app:implementation}.
Our code and intermediate artifacts (e.g., constructed taxonomies) are available at:
\url{https://github.com/Cloudy1225/TIER}.

\begin{table*}[!t]
\centering
\caption{Node classification accuracy on eight TRN datasets. 
The \colorbox{orange!25}{\textbf{best}} and \colorbox{orange!10}{second-best} results are highlighted.}
\resizebox{\linewidth}{!}{
    \begin{tabular}{l|cccccccc|c}
    \toprule
    \rowcolor{COLOR_MEAN} Model & Cora & Citeseer & Pubmed & ArXiv & WikiCS & Books & Photo & Computer & Avg. \\ \midrule
    GCN$_{\text{SE}}$~\cite{GCN} & 82.30±0.19 & 70.55±0.32 & 78.94±0.27 & 71.39±0.28 & 79.86±0.19 & 68.79±0.46 & 69.25±0.81 & 71.44±1.19 & 74.06 \\
    SAGE$_{\text{SE}}$~\cite{SAGE} & 82.27±0.37 & 69.56±0.43 & 77.88±0.44 & 71.21±0.18 & 79.67±0.25 & 72.01±0.33 & 78.50±0.15 & 81.43±0.27 & 76.57 \\
    GAT$_{\text{SE}}$~\cite{GAT} & 81.30±0.67 & 69.94±0.74 & 78.49±0.70 & 71.57±0.25 & 79.99±0.65 & 74.35±0.35 & 80.40±0.45 & 83.39±0.22 & 77.43 \\ \midrule
    SenBERT~\cite{SentenceBERT} & 66.66±1.42 & 60.52±1.62 & 36.04±2.92 & 72.66±0.24 & 77.77±0.75 & 83.68±0.19 & 73.89±0.31 & 70.76±0.15 & 67.75 \\
    RoBERTa~\cite{RoBERTa} & 72.24±1.14 & 66.68±2.03 & 42.32±1.56 & 74.12±0.12 & 76.81±1.04 & 84.62±0.16 & 74.79±1.13 & 72.31±0.37 & 70.49 \\ \midrule
    GIANT~\cite{GIANT} & 81.08±1.07 & 66.02±1.28 & 79.32±0.45 & 74.29±0.10 & 80.48±0.65 & 81.99±0.37 & 71.83±0.38 & 86.16±0.12 & 77.65 \\
    GLEM~\cite{GLEM} & 81.30±0.88 & 68.80±2.46 & \cellcolor{orange!10} 81.70±1.07 & 73.55±0.22 & 76.43±0.55 & 83.28±0.39 & 76.93±0.49 & 80.46±1.45 & 77.81 \\ \midrule
    ENGINE~\cite{ENGINE} & \cellcolor{orange!10} 84.22±0.46 & \cellcolor{orange!10} 72.14±0.74 & 77.84±0.27 & 74.69±0.36 & 80.94±0.19 & 82.89±0.14 & 84.33±0.57 & 86.42±0.23 & 80.43 \\
    GCN$_{\text{LLM}}$~\cite{LLMNodeBed} & 83.33±0.75 & 71.39±0.90 & 78.71±0.45 & 74.39±0.31 & 80.94±0.16 & 83.03±0.34 & 84.84±0.47 & 88.22±0.16 & 80.61 \\ \midrule
    TAPE~\cite{TAPE} & 84.04±0.24 & 71.87±0.35 & 78.61±1.23 & 74.96±0.14 & \cellcolor{orange!10} 81.94±0.16 & \cellcolor{orange!10} 84.92±0.26 & \cellcolor{orange!10} 86.46±0.12 & \cellcolor{orange!10} 89.52±0.04 & \cellcolor{orange!10} 81.54 \\
    GAugLLM~\cite{GAugLLM} & 82.84±0.99 & 70.19±0.92 & 80.59±0.82 & 73.59±0.10 & 80.32±0.32 & 82.12±0.39 & 76.39±0.62 & 87.79±0.23 & 79.23 \\ \midrule
    GraphGPT~\cite{GraphGPT} & 64.72±1.50 & 64.58±1.55 & 70.34±2.27 & 75.15±0.14 & 70.71±0.37 & 81.13±1.52 & 77.48±0.78 & 80.10±0.76 & 73.03 \\
    LLaGA~\cite{LLaGA} & 78.94±1.14 & 62.61±3.63 & 65.91±2.09 & 74.49±0.23 & 76.47±2.20 & 83.47±0.45 & 84.44±0.90 & 87.82±0.53 & 76.77 \\
    LLM$_{\text{IT}}$~\cite{LLMNodeBed} & 67.00±0.16 & 54.26±0.22 & 80.99±0.43 & \cellcolor{orange!25} \textbf{76.08±0.00} & 75.02±0.16 & 80.92±1.38 & 71.28±1.81 & 66.99±2.02 & 71.57 \\ \midrule
    \textbf{TIER} & \cellcolor{orange!25} \textbf{84.89±0.38} & \cellcolor{orange!25} \textbf{73.70±0.92} & \cellcolor{orange!25} \textbf{81.83±1.11} & \cellcolor{orange!10} 75.96±0.13 & \cellcolor{orange!25} \textbf{82.51±0.25} & \cellcolor{orange!25} \textbf{85.15±0.21} & \cellcolor{orange!25} \textbf{87.16±0.19} & \cellcolor{orange!25} \textbf{89.80±0.13} & \cellcolor{orange!25} \textbf{82.62} \\
    \bottomrule
    \end{tabular}
}
\label{tab:mainexp}
\end{table*}

\subsection{Experimental Results}

\subsubsection{Performance Comparison (RQ1)}
Table~\ref{tab:mainexp} presents the classification accuracy of our proposed method TIER compared to a broad range of baselines. 
We observe that methods leveraging graph structure, even with shallow node embeddings, consistently outperform those relying solely on textual input, such as SenBERT, RoBERTa, and LLM$_\text{IT}$. 
This highlights the importance of structural information in TRNs.
Our method achieves the best performance on nearly all datasets, demonstrating the effectiveness of capturing hierarchical semantic structure. 
Notably, TIER relies solely on RoBERTa for node text encoding, leading to significantly fewer parameters and improved scalability compared to LLM-based baselines. 
While TAPE achieves the second-best overall performance, it requires prompting an LLM to generate lengthy explanations for every single node, which is both time- and cost-intensive. 
In contrast, our approach prompts the LLM only for selected clusters and outliers, drastically reducing the number of queries. 
On the ArXiv dataset, our performance is slightly lower than LLM$_\text{IT}$, which may be attributed to the limited text understanding capability of RoBERTa compared to powerful instruction-tuned LLMs. 
Overall, by capturing hierarchical structures through taxonomy construction and regularization, TIER provides a more semantically grounded representation of nodes in TRNs and consistently outperforms even LLM-based methods across a wide range of datasets.

\begin{figure*}[t]
\centering
\subfigure[Without Taxonomy]{\includegraphics[width=0.26\linewidth]{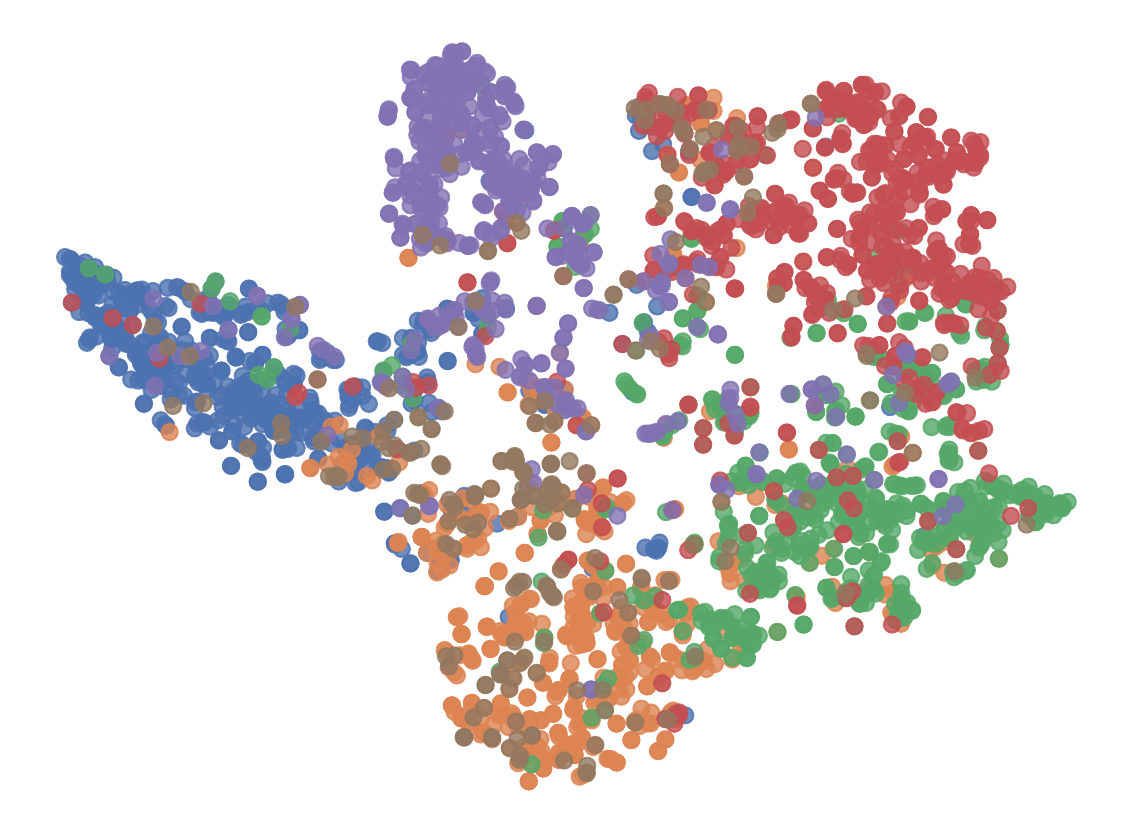}} 
\subfigure[Without Taxonomy]{\includegraphics[width=0.19\linewidth]{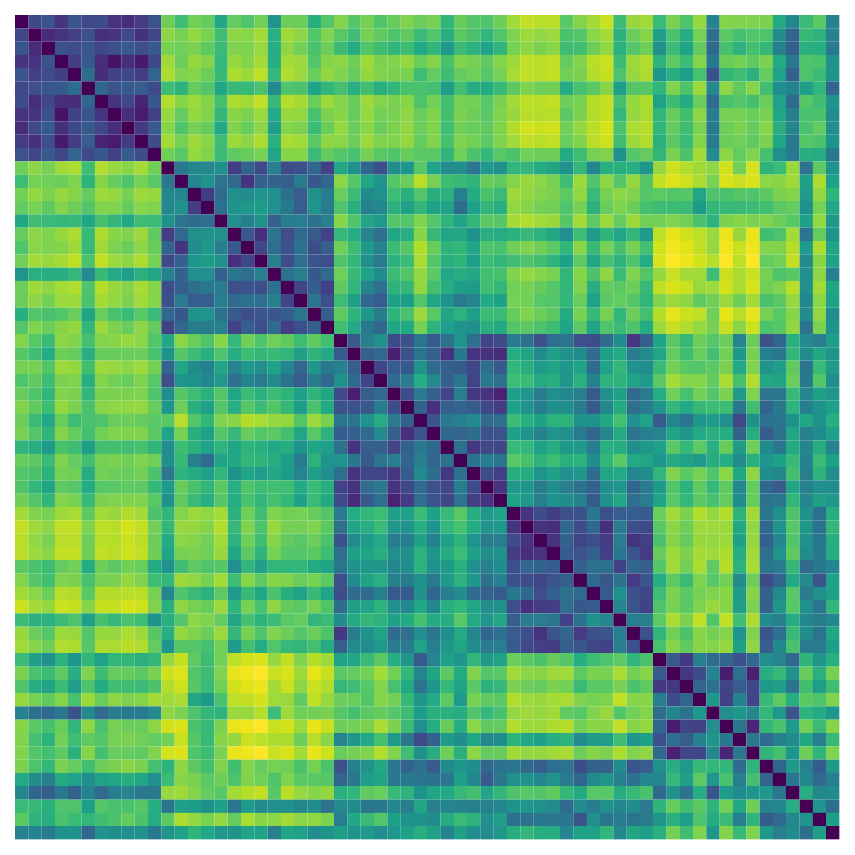}}
\subfigure[With Taxonomy]{\includegraphics[width=0.26\linewidth]{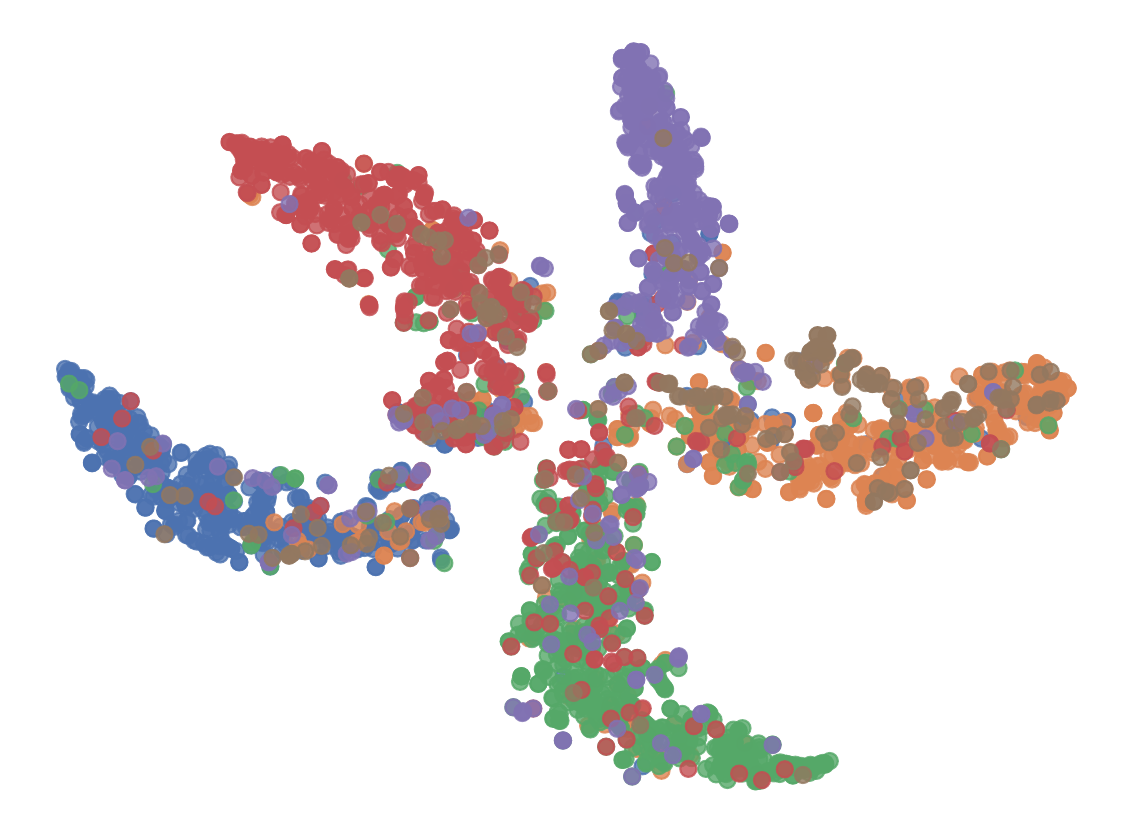}} 
\subfigure[With Taxonomy]{\includegraphics[width=0.19\linewidth]{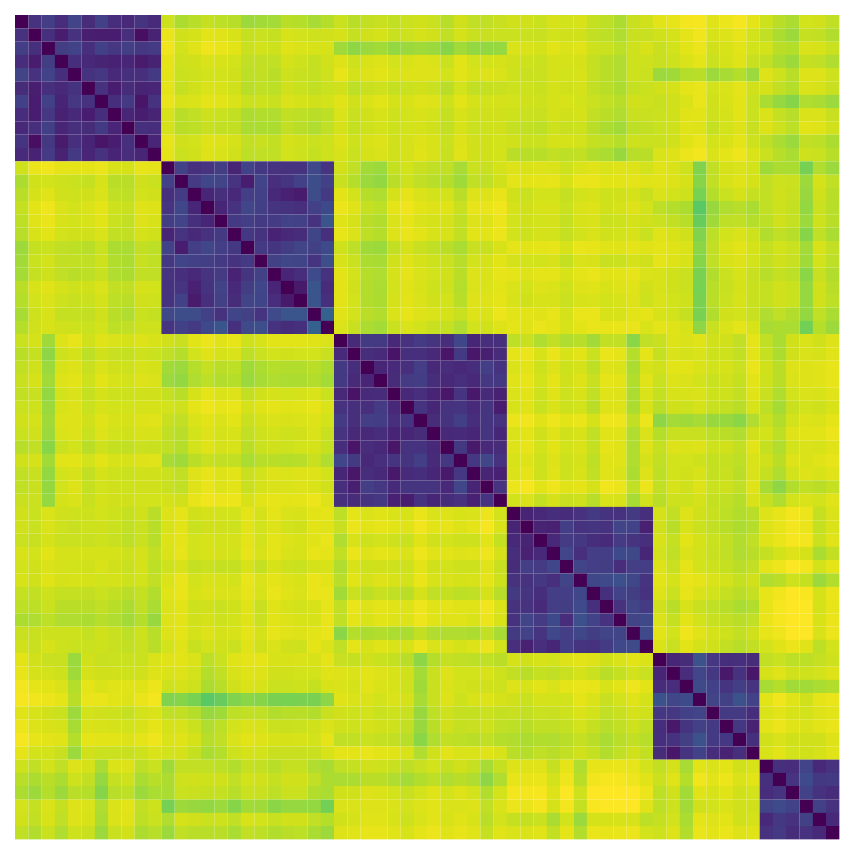}}
\caption{Visualizations of the learned node representations (colored by ground-truth labels) and the pairwise distance matrix between finest-level cluster centroids on Citeseer, with and without taxonomy regularization. 
The constructed taxonomy tree has 3 hierarchical levels with 1, 6, and 62 cluster nodes from top to bottom, respectively. Darker colors indicate smaller distances. With taxonomy regularization, clearer block structures emerge, where darker diagonal blocks correspond to coarser-grained clusters, reflecting improved semantic hierarchy in learned representation space.
}
\label{fig:visual_citeseer}
\end{figure*}

\subsubsection{Representation Visualization (RQ2)}
To assess whether TIER effectively learns a hierarchically structured representation space, we visualize the node embeddings learned with and without taxonomy-informed regularization. 
Taking Citeseer as an illustrative example, the constructed taxonomy tree consists of three levels, containing 1, 6, and 62 cluster nodes from top to bottom.
We first extract node representations from the last hidden layer of the GNN encoder and compute the finest-level cluster centroids via mean pooling over the nodes belonging to each cluster. 
We then compute the pairwise distance matrix between these centroids and visualize it in Figure~\ref{fig:visual_citeseer} (b) and (d), where darker colors indicate smaller distances.
From the distance matrices, we observe that taxonomy-informed regularization leads to significantly clearer block structures, with prominent dark diagonal blocks representing coarse-grained clusters. 
This pattern reflects improved semantic organization, as the intra-coarse-cluster distances are significantly smaller than inter-coarse-cluster distances (yielding six diagonal blocks). 
This suggests that the model learns a representation space where fine-grained clusters are embedded closer if they share a higher-level semantic category.
We further apply t-SNE~\cite{t-SNE} to project the learned node embeddings into 2D space.
As shown in Figure~\ref{fig:visual_citeseer} (a) and (c), the embeddings learned with taxonomy regularization indeed exhibit better separation of coarse-grained semantic classes.
In summary, taxonomy-informed regularization enables our model to learn a semantically meaningful and hierarchically interpretable representation space, capturing both fine-grained and coarse-grained semantics, which in turn improves class separability.
Additional visualization results on other datasets can be found in Appendix~\ref{app:representation_visualization}.


\begin{table*}[h]
\centering
\caption{Efficiency comparison in terms of training time and GPU memory (GB) on ArXiv.}
    \begin{tabular}{l|cccccccccc}
    \toprule
    \rowcolor{COLOR_MEAN} ~ & GCN$_{\text{SE}}$ & SenBERT & RoBERTa & ENGINE & GCN$_{\text{LLM}}$ & TAPE & GraphGPT & LLaGA & LLM$_{\text{IT}}$ & TIER \\ \midrule
    Mem. & 4.79 & 9.98 & 46.10 & 7.08 & 7.36 & 46.10 & 60.41 & 41.94 & 69.13 & 6.78 \\
    Time & 51.2s & 7.4m & 40.8m & 2.6h & 1.4h & 37.4h & 7.8h & 7.7h & 36.3h & 16.8m \\
    \bottomrule
    \end{tabular}
\label{tab:efficiency}
\end{table*}

\subsubsection{Taxonomy Visualization (RQ3)}

\begin{figure}[h]
\centerline{\includegraphics[width=1.\linewidth]{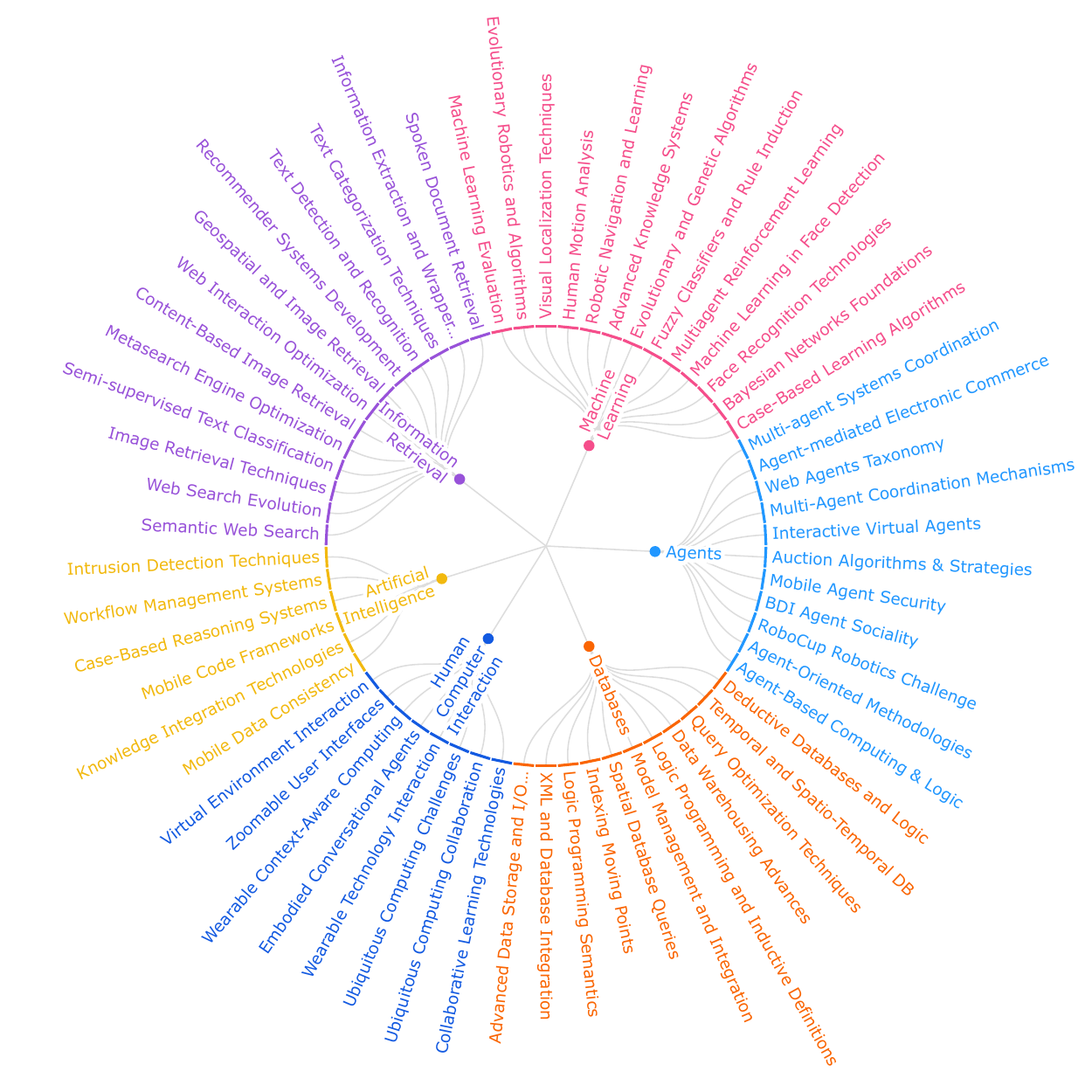}}\caption{RadialMap of the constructed taxonomy on Citeseer.}\label{fig:citeseer_taxonomy_small}
\end{figure}

In this subsection, we present a visualization of the constructed taxonomy tree, where each node is described using cluster labels generated by the LLM (see Section~\ref{sec:llm_hierarchical_clustering}). 
The RadialMap of the taxonomy for Citeseer is shown in Figure~\ref{fig:citeseer_taxonomy_small}; a larger and clearer version is provided in Figure~\ref{fig:citeseer_taxonomy}.
As observed, it is evident that the constructed taxonomy is of high quality and effectively captures hierarchical semantic relationships. 
For instance, the "Agents" coarse cluster includes finer sub-clusters whose labels almost universally contain the term "Agent", such as "Mobile Agent Security" and "BDI Agent Sociality", which correspond to various subfields within Agent-based Systems. 
Similarly, the "Information Retrieval" coarse cluster encompasses sub-clusters like Web Search, Image Retrieval, Recommender Systems, and Text Classification, which, although they might be categorized under "Machine Learning", are indeed widely used in document retrieval~\cite{TC4IR}, justifying their grouping. 
One subtle case is "Embodied Conversational Agents" (ECAs), which was placed under the broader "Human-Computer Interaction" cluster. 
While the model might have alternatively categorized it under "Agents", this assignment is semantically reasonable as ECAs are a form of intelligent user interface that facilitates rich human-computer interaction~\cite{ECA}.
Overall, these examples illustrate that our constructed taxonomy is semantically consistent, reflects fine-to-coarse relationships, and aligns well with domain-specific knowledge. 
Visualizations of taxonomies constructed on other datasets can be found in Appendix~\ref{app:taxonomy_visualization}.

\subsubsection{Efficiency Comparison (RQ4)}

To assess the scalability of our approach, we evaluate the training time and GPU memory usage of TIER on the large-scale ArXiv dataset. 
The results are summarized in Table~\ref{tab:efficiency}. 
Note that baseline results are directly taken from LLMNodeBed~\cite{LLMNodeBed}, where all methods were run on a single NVIDIA H100-80GB GPU.
In contrast, TIER is evaluated on a single NVIDIA Tesla V100 32G GPU, which has significantly lower memory and computational capacity.
Despite this hardware disadvantage, TIER demonstrates strong efficiency. 
It achieves the second-lowest memory footprint, using only 6.78 GB, which is significantly smaller than LLM-based methods such as GraphGPT and LLM$_\text{IT}$, both of which require more than 60 GB of GPU memory.
In terms of training time, TIER ranks third-fastest, completing training in 16.8 minutes, which is at least 5× faster than most LLM-based approaches. 
For instance, TAPE and LLM$_\text{IT}$ require over 36 hours, showing the computational overhead of instruction tuning or prompt-based generation for each node.
Overall, TIER achieves superior accuracy (as shown in Table~\ref{tab:mainexp}) while maintaining high efficiency in both time and memory, making it an attractive solution for real-world deployment without incurring excessive resource costs.

\subsubsection{Ablation Study (RQ5)}

We conduct ablation experiments to evaluate the contribution of key components in both the taxonomy construction and representation learning stages of TIER. Specifically, we consider the following three variants: 
\emph{TIER w/o SGCL} replaces the similarity-guided contrastive loss Eq.~\eqref{eq:sgcl} with supervised contrastive learning (SupCon)~\cite{SupCon}, removing the influence of graph structure in pertaining.
\emph{TIER w/o LLM} removes the LLM-assisted clustering refinement during taxonomy construction and instead uses pure hierarchical K-Means without semantic validation or correction. 
\emph{TIER w/o CCC} disables the taxonomy-informed regularization, optimizing the GCN solely with cross-entropy loss, thus ignoring hierarchical constraints.
Table~\ref{tab:ablation_study} reports the performance of these variants across four representative datasets from diverse domains. We observe the following:
1) Removing SGCL leads to consistent performance drops, especially on WikiCS and Photo. This confirms that incorporating structural similarity via graph-based contrastive learning is beneficial for producing cluster-friendly embeddings.
2) While hierarchical K-Means alone still yields reasonable results, removing LLM refinement causes noticeable degradation in performance, particularly on WikiCS. This indicates that LLM-guided refinement improves the coherence and semantic alignment of the taxonomy, enhancing its utility in downstream learning.
3) Removing CCC regularization results in the most significant drop. This demonstrates that encouraging the representation space to reflect the hierarchical distances among categories provides strong guidance for semantic structuring and improves classification performance.
Overall, each component contributes to the final effectiveness of TIER, and their synergy is essential for fully exploiting hierarchical signals in TRNs.

\begin{table}[h]
\centering
\caption{Ablation study results for TIER.}
    \begin{tabular}{l|cccc}
    \toprule
    \rowcolor{COLOR_MEAN} Variant & Cora & Citeseer & WikiCS & Photo \\ \midrule
    TIER & \textbf{84.89±0.38} & \textbf{73.70±0.92} & \textbf{82.51±0.25} & \textbf{87.16±0.19} \\
    w/o SGCL & 83.40±0.30 & 72.46±1.24 & 79.79±0.14 & 85.01±0.28 \\
    w/o LLM & 83.47±0.71 & 72.95±1.07 & 80.33±0.29 & 86.53±0.35 \\
    w/o CCC & 83.20±0.62 & 71.63±1.60 & 80.18±0.34 & 85.44±0.22 \\
    \bottomrule
    \end{tabular}
\label{tab:ablation_study}
\end{table}

\vpara{Additional Analysis.}
Due to space constraints, the hyperparameter analysis and additional figures are deferred in Appendix~\ref{app:hyperparameter_analysis}.


\section{Conclusion}
This paper introduces TIER, a taxonomy-informed framework for hierarchical knowledge learning on TRNs.
By constructing a taxonomy through LLM-powered hierarchical clustering and leveraging it to regularize node representations via the Cophenetic Correlation Coefficient, TIER effectively captures both fine-grained and coarse-grained semantics. 
Extensive experiments on eight benchmark datasets demonstrate the superiority of TIER, highlighting the value of learning hierarchical knowledge in TRNs.

\begin{acks}
This work is partially supported by the National Key Research and Development Program of China (2024YFB2505604), the National Natural Science Foundation of China (62306137), and the Ant Group Research Intern Program.
\end{acks}

\bibliographystyle{ACM-Reference-Format}
\bibliography{main}

\appendix

\section{Proof}\label{app:proof}

\begin{proof}
\vpara{Notation.}
To facilitate proof, let:
\begin{itemize}
    \item $n$: number of nodes; $k$: number of classes.
    
    \item $\boldsymbol{Y} \in \{0,1\}^{n \times k}$: the one-hot label matrix of all nodes.
    
    \item $\boldsymbol{Y}_l \in \{0,1\}^{n \times k}$: the partially observed label matrix, where unlabeled rows are all zeros.
    
    \item $\boldsymbol{S}^* := \boldsymbol{Y} \boldsymbol{Y}^\top \in \{0,1\}^{n \times n}$: the ground-truth pairwise similarity matrix, where $S^*_{ij} = 1$ if nodes $i$ and $j$ are in the same class.
    
    \item $\boldsymbol{A} \in \{0,1\}^{n \times n}$: the undirected binary adjacency matrix.
    \item $\boldsymbol{I}$: identity matrix of size $n \times n$.
    
    \item $h \in [0,1]$: edge homophily, defined as the proportion of edges that connect nodes of the same class.
\end{itemize}

We define the following similarity matrices:

\begin{itemize}
    \item $\boldsymbol{S}^1 = \boldsymbol{I}$: Identity-based similarity used in vanilla contrastive learning~\cite{SimCLR, SCL}, where each node is only pulled close to itself.

    \item $\boldsymbol{S}^2 = \boldsymbol{I} + \boldsymbol{Y}_l \boldsymbol{Y}_l^\top$: SupCon~\cite{SupCon}-style similarity, where node pairs with known identical labels are treated as positives.

    \item $\boldsymbol{S}^3 = \boldsymbol{I} + \boldsymbol{Y}_l \boldsymbol{Y}_l^\top + \boldsymbol{A}$: Our full construction, which incorporates identity, known label agreement, and graph connectivity.
\end{itemize}

We quantify the distance between a constructed similarity matrix $\boldsymbol{S}$ and the ground-truth matrix $\boldsymbol{S}^*$ using Frobenius norm:
\[
\mathcal{E}(\boldsymbol{S}) := \|\boldsymbol{S} - \boldsymbol{S}^*\|_F^2 = \sum_{i,j} (S_{ij} - S^*_{ij})^2
\]

\noindent Our goal is to show:
\[
\mathcal{E}(\boldsymbol{S}^3) < \mathcal{E}(\boldsymbol{S}^2) < \mathcal{E}(\boldsymbol{S}^1)
\]

\noindent and further prove that minimizing the contrastive loss using $\boldsymbol{S}^3$ leads to better cluster-preserving node representations.

\vpara{Step 1: $\boldsymbol{S}^2$ improves upon $\boldsymbol{S}^1$.}

Compared to $\boldsymbol{S}^1 = \boldsymbol{I}$, $\boldsymbol{S}^2$ adds 1s at positions $(i,j)$ where both $i$ and $j$ are labeled and share the same label ($Y_{l,i} = Y_{l,j} \neq \boldsymbol{0}$). Since $\boldsymbol{Y}_l \boldsymbol{Y}_l^\top \subset \boldsymbol{Y} \boldsymbol{Y}^\top$, these additions are consistent with the ground-truth $\boldsymbol{S}^*$.

Thus, $\boldsymbol{S}^2$ correctly fixes a subset of false negatives in $\boldsymbol{S}^1$, reducing the approximation error without introducing false positives:
\[
\mathcal{E}(\boldsymbol{S}^2) < \mathcal{E}(\boldsymbol{S}^1)
\]

\vpara{Step 2: $\boldsymbol{S}^3$ improves upon $\boldsymbol{S}^2$ under $h > 0.5$.}

Let $\mathcal{P} := \{(i,j) \mid A_{ij} = 1\}$ denote all connected pairs. Partition $\mathcal{P}$ into:
\[
\mathcal{P}_{\text{same}} := \{(i,j) \in \mathcal{P} \mid Y_i = Y_j\}, \quad
\mathcal{P}_{\text{diff}} := \mathcal{P} \setminus \mathcal{P}_{\text{same}}
\]

The homophily assumption gives:
\[
h := \frac{|\mathcal{P}_{\text{same}}|}{|\mathcal{P}|} > 0.5 \Rightarrow |\mathcal{P}_{\text{same}}| > |\mathcal{P}_{\text{diff}}|
\]

Adding $\boldsymbol{A}$ to $\boldsymbol{S}^2$ adds 1s at all positions in $\mathcal{P}$. This has the following effects:
\begin{itemize}
    \item For $(i,j) \in \mathcal{P}_{\text{same}}$, $S^3_{ij} = 1 = S^*_{ij}$: error is reduced.
    \item For $(i,j) \in \mathcal{P}_{\text{diff}}$, $S^3_{ij} = 1 \neq S^*_{ij} = 0$: error is increased.
\end{itemize}

Thus, the change in approximation error is:
\[
\Delta \mathcal{E} = |\mathcal{P}_{\text{diff}}| - |\mathcal{P}_{\text{same}}| < 0
\]

Therefore:
\[
\mathcal{E}(\boldsymbol{S}^3) < \mathcal{E}(\boldsymbol{S}^2)
\]

\vpara{Step 3: Effect on Contrastive Representation Learning.}

The similarity-guided contrastive loss is defined as:
\[
\mathcal{L}_{\text{SCL}} = -\sum_{S_{ij}=1} \boldsymbol{z}_i^\top \boldsymbol{z}_j + \gamma \sum_{S_{ij}=0} (\boldsymbol{z}_i^\top \boldsymbol{z}_j)^2
\]

Optimizing this loss encourages:
\begin{itemize}
    \item High similarity between $\boldsymbol{z}_i$ and $\boldsymbol{z}_j$ when $S_{ij}=1$
    \item Low similarity when $S_{ij}=0$
\end{itemize}

The minimizer $\boldsymbol{Z}$ of this loss will satisfy:
\[
\boldsymbol{Z} \boldsymbol{Z}^\top \approx \boldsymbol{S}
\]

Therefore, a better approximation of $\boldsymbol{S}^*$ by $\boldsymbol{S}$ directly translates to representations $\boldsymbol{Z}$ that more closely reflect ground-truth semantic clusters.

Since we have shown:
\[
\mathcal{E}(\boldsymbol{S}^3) < \mathcal{E}(\boldsymbol{S}^2) < \mathcal{E}(\boldsymbol{S}^1)
\]
we conclude that $\boldsymbol{Z}$ learned using $\boldsymbol{S}^3$ better preserve semantic clusters, compared to those obtained using $\boldsymbol{S}^2$ or $\boldsymbol{S}^1$.

\end{proof}

\section{Prompts}
This section shows the prompts used in our LLM-assisted clustering refinement (see Section~\ref{sec:llm_hierarchical_clustering} for more information).

\setlength{\intextsep}{5pt} 
\setlength{\textfloatsep}{5pt} 
\setlength{\floatsep}{5pt} 
\setlength{\abovecaptionskip}{5pt} 
\setlength{\belowcaptionskip}{5pt} 
\begin{prompt}[H]
\caption{The prompt for cluster splitting.}
\label{pr:split}
\begin{tcolorbox}[colback=gray!10, colframe=black, boxrule=1pt, arc=2pt, left=5pt, right=5pt]
Here are some documents currently in one cluster:

$\newline$
1. \textcolor{blue}{\textbf{\{Document 1\}}}

2. \textcolor{blue}{\textbf{\{Document 2\}}}

3. \textcolor{blue}{\textbf{\{Document 3\}}}

...

$\newline$
Please decide whether they should stay in one cluster, or be split into multiple subclusters based on different topics. Only respond with a single number: the number of subclusters needed. If no split is needed, respond with 1. 
\end{tcolorbox}
\end{prompt}

\begin{prompt}[H]
\caption{The prompt for cluster merging.}
\label{pr:merge}
\begin{tcolorbox}[colback=gray!10, colframe=black, boxrule=1pt, arc=2pt, left=5pt, right=5pt]
Here are two clusters of documents:

$\newline$
Cluster A:

1. \textcolor{blue}{\textbf{\{Document A1\}}}

2. \textcolor{blue}{\textbf{\{Document A2\}}}

3. \textcolor{blue}{\textbf{\{Document A3\}}}

...

$\newline$
Cluster B:

1. \textcolor{blue}{\textbf{\{Document B1\}}}

2. \textcolor{blue}{\textbf{\{Document B2\}}}

3. \textcolor{blue}{\textbf{\{Document B3\}}}

...

$\newline$
Determine whether the two clusters are about the same or highly similar topic and should be merged. Return only 1 if they should be merged, or 0 if they should remain separate.
\end{tcolorbox}
\end{prompt}

\begin{prompt}[H]
\caption{The prompt for cluster labeling and summarization.}
\label{pr:summarize}
\begin{tcolorbox}[colback=gray!10, colframe=black, boxrule=1pt, arc=2pt, left=5pt, right=5pt]
Here are some documents from the same cluster:

$\newline$
1. \textcolor{blue}{\textbf{\{Document 1\}}}

2. \textcolor{blue}{\textbf{\{Document 2\}}}

3. \textcolor{blue}{\textbf{\{Document 3\}}}

...

$\newline$
Please:

- Generate a short topic label (2–5 words)

- Summarize the common theme in 1–2 sentences

$\newline$
Output format (in JSON):

\{

  "label": "[label]",
  
  "summary": "[summary]"
  
\}
\end{tcolorbox}
\end{prompt}

\section{Additional Experimental Details}
\subsection{Datasets}~\label{app:datasets}
To provide a comprehensive evaluation of our proposed method, we adopt 8 datasets from the LLMNodeBed~\cite{LLMNodeBed} benchmark, as summarized in Table~\ref{tab:dataset_statistics}. 
These datasets span diverse domains, including academic networks, web graphs, and product co-purchase networks, and cover a wide range of scales, from thousands to hundreds of thousands of nodes. 
Within LLMNodeBed, each dataset is stored in \texttt{.pt} format using PyTorch, which includes shallow embeddings, raw text of nodes, edge indices, labels, and data splits for convenient loading. 
The processed data is publicly available at \url{https://huggingface.co/datasets/xxwu/LLMNodeBed}. 
A description of these datasets is provided below:

\begin{itemize}
    \item \textbf{Academic Networks:} The \textbf{Cora}, \textbf{Citeseer}, \textbf{Pubmed}~\cite{GCN}, and \textbf{ogbn-arXiv} (abbreviated as "ArXiv")~\cite{OGB} datasets consist of nodes representing papers, with edges indicating citation relationships. 
    The associated text attributes include each paper's title and abstract, which we use the collected version as follows: Cora and Pubmed from~\citet{TAPE}, Citeseer from~\citet{LLM4Graph3}. 
    Within the dataset, each node is labeled according to its category. 
    For example, the ArXiv dataset includes 40 CS sub-categories such as cs.AI (Artificial Intelligence) and cs.DB (Databases).

    \item \textbf{Web Link Network:} In the \textbf{WikiCS} dataset~\cite{WikiCS}, each node represents a Wikipedia page, and edges indicate reference links between pages. 
    The raw text for each node includes the page name and content, which was collected by \citet{OFA}. 
    The classification goal is to categorize each entity into different Wikipedia categories.

    \item \textbf{E-Commerce Networks:} The \textbf{Ele-Photo} (abbreviated as "Photo") and \textbf{Ele-Computer} (abbreviated as "Computer") datasets are derived from the Amazon Electronics dataset~\cite{Amazon}, where each node represents an item in the Photo or Computer category. The \textbf{Books-History} (abbreviated as "Books") dataset comes from the Amazon Books dataset, where each node corresponds to a book in the history category. 
    We utilize the processed datasets released in~\citet{CS-TAG}. 
    In these e-commerce networks, edges indicate co-purchase or co-view relationships. The associated text for each item includes descriptions, e.g., book titles and summaries, or user reviews. 
    The classification task involves categorizing these products into fine-grained sub-categories.
\end{itemize}

\begin{prompt}[t]
\caption{The prompt for outlier reassignment.}
\label{pr:outlier}
\begin{tcolorbox}[colback=gray!10, colframe=black, boxrule=1pt, arc=2pt, left=5pt, right=5pt]
Here is a document and a list of cluster summaries:

$\newline$
Document:

\textcolor{blue}{\textbf{\{Text\}}}

$\newline$
Cluster 1:

Label: \textcolor{blue}{\textbf{\{Label 1\}}}

Summary: \textcolor{blue}{\textbf{\{Summary 1\}}}

$\newline$
Cluster 2:

Label: \textcolor{blue}{\textbf{\{Label 2\}}}

Summary: \textcolor{blue}{\textbf{\{Summary 2\}}}

$\newline$
Cluster 3:

Label: \textcolor{blue}{\textbf{\{Label 3\}}}

Summary: \textcolor{blue}{\textbf{\{Summary 3\}}}

...

$\newline$
Decide which cluster the document best belongs to. Only return the number of the best matching cluster (e.g., 1, 2, 3, ...).
\end{tcolorbox}
\end{prompt}

\begin{table}[t!]
    \begin{center}
    \setlength{\tabcolsep}{3.5pt}
    {\caption{Dataset statistics including the number of nodes and edges, the average number of tokens in the textual descriptions associated with nodes, the number of classes, and the domains they belong to.}\label{tab:dataset_statistics}}
    \begin{tabular}{lccccc}
    \toprule
    Dataset & \#Nodes & \#Edges & \#Tokens & \#Classes & Domain \\
    \midrule 
    Cora & 2,708 & 5,429 & 183.4 & 7 & CS Citation \\
    Citeseer & 3,186 & 4,277 & 210.0 & 6 & CS Citation \\
    Pubmed & 19,717 & 44,338 & 446.5 & 3 & Bio Citation \\
    ArXiv & 169,343 & 1,166,243 & 239.8 & 40 & CS Citation \\
    WikiCS & 11,701 & 216,123 & 629.9 & 10 & Web Link \\
    Books & 41,551 & 358,574 & 337.0 & 12 & E-Commerce \\
    Photo & 48,362 & 500,928 & 201.5 & 12 & E-Commerce \\
    Computer & 87,229 & 721,081 & 123.1 & 10 & E-Commerce \\
    \bottomrule
    \end{tabular}
    \end{center}
\end{table}

\subsection{Implementation}~\label{app:implementation}
All experiments are conducted using the most up-to-date benchmark suite LLMNodeBed~\cite{LLMNodeBed} for TRNs to date. 
Our task is semi-supervised node classification, using the dataset splits and shallow input embeddings provided by LLMNodeBed.
To ensure fair comparisons, all baseline methods are implemented with consistent architectural components: GCN~\cite{GCN} as the backbone GNN, RoBERTa (355M)~\cite{RoBERTa} as the PLM, and Mistral-7B~\cite{Mistral7B} as the default LLM where applicable.
Each experiment is repeated across 5 random seeds, and we report the mean accuracy and standard deviation.
For our proposed method, initial node embeddings are obtained using frozen RoBERTa (355M) on raw text. 
The taxonomy construction and downstream node classification are trained with separate GCN encoders. 
To refine the hierarchical clusters, we prompt the DeepSeek-V3~\cite{DeepSeekV3} via API. 
Following LLMNodeBed, we search GCN hyperparameters over: number of layers $\in \{2, 3, 4\}$, hidden dimensions $\in \{64, 128, 256\}$, dropout rates $\in \{0.3, 0.5, 0.7\}$, and consider configurations with and without batch normalization and residual connections. 
The trade-off weights $\gamma$ and $\lambda$ of the taxonomy-informed regularization loss are set to 1.0 by default. 
For LLM-Assisted Clustering Refinement, we use the following hyperparameters:
$\tau_{\text{split}} = 0.75$, $n_{\text{split}} = 20$, $\tau_{\text{merge}} = 0.9$, $n_{\text{merge}} = 10$, $n_{\text{min}} = 10$, $n_{\text{close}} = 10$, $r = 0.05$, and $n_{\text{outlier}} = 3$.
The shape of the constructed taxonomy tree (i.e., its depth and number of nodes per level) is primarily determined by the dataset size and granularity. The default settings for each dataset are: Cora (1, 7, 64), Citeseer (1, 6, 64), Pubmed (1, 3, 16, 64), WikiCS (1, 10, 128), Books (1, 12, 64, 256), Photo (1, 12, 64, 256), Computer (1, 10, 128, 512), ArXiv (1, 40, 128, 512, 2048).
Note that the final shape of the taxonomy tree may deviate slightly from the target due to LLM-based refinement (e.g., cluster splits and merges).
We run our method on a Linux server equipped with an Intel(R) Xeon(R) Gold 6248 CPU @ 2.50GHz and a 32GB NVIDIA Tesla V100 GPU. 
Our code and intermediate artifacts (e.g., constructed taxonomies) are available at:
\url{https://github.com/Cloudy1225/TIER}.

\begin{table*}[t]
\centering
\caption{Performance with different taxonomy tree structures.}
\label{tab:taxonomy_shape}
\begin{tabular}{l|cccccc}
\toprule
\multirow{2}{*}{Cora} & (1, 3, 64) & (1, 7, 64) & (1, 14, 64) & (1, 21, 64) & (1, 4, 16, 64) & (1, 7, 21, 64) \\
                      & 82.58±1.28 & \textbf{84.89±0.38} & 83.61±2.34 & 83.31±0.43 & 82.69±1.47 & 84.26±1.41 \\
\midrule
\multirow{2}{*}{Citeseer} & (1, 3, 64) & (1, 6, 64) & (1, 12, 64) & (1, 18, 64) & (1, 4, 16, 64) & (1, 6, 18, 64) \\
                          & 72.16±0.921 & \textbf{73.70±0.92} & 71.72±1.29 & 72.10±1.81 & 71.55±1.084 & 73.36±1.32 \\
\midrule
\multirow{2}{*}{WikiCS} & (1, 5, 128) & (1, 10, 128) & (1, 20, 128) & (1, 30, 128) & (1, 10, 32, 128) & (1, 10, 64, 128) \\
                        & 79.83±0.04 & \textbf{82.51±0.25} & 80.62±0.37 & 80.83±0.38 & 82.09±0.33 & 81.94±0.68 \\
\midrule
\multirow{2}{*}{Photo} & (1, 6, 256) & (1, 12, 256) & (1, 24, 256) & (1, 36, 256) & (1, 12, 32, 256) & (1, 12, 64, 256) \\
                       & 85.73±0.21 & 86.61±0.209 & 86.32±0.25 & 86.06±0.21 & 86.16±0.23 & \textbf{87.16±0.19} \\
\bottomrule
\end{tabular}
\end{table*}

\subsection{Taxonomy Visualization}\label{app:taxonomy_visualization}
Figures \ref{fig:citeseer_taxonomy}, \ref{fig:pubmed_taxonomy}, \ref{fig:wikics_taxonomy}, and \ref{fig:photo_taxonomy} visualize the constructed taxonomy on Citeseer, Pubmed, WikiCS, and Photo, respectively. Due to the large shape of the constructed taxonomy on Photo, we present 5 clusters at the second level.

\subsection{Representation Visualization}\label{app:representation_visualization}
Figures \ref{fig:visual_pubmed}, \ref{fig:visual_wikics}, and \ref{fig:visual_photo} visualize the learned node representations on Pubmed, WikiCS, and Photo, respectively.

\subsection{Hyperparameter Analysis}\label{app:hyperparameter_analysis}

\vpara{On the trade-off weight $\lambda$.}
We conduct a sensitivity analysis on the hyperparameter $\lambda$, which controls the strength of the taxonomy-informed regularization loss in Eq.\eqref{eq:loss}.
Figure~\ref{fig:lambda} reports the performance of TIRAG under varying values of $\lambda$ across four representative datasets.
We observe that introducing the regularization term ($\lambda > 0$) consistently improves performance over the non-regularized variant ($\lambda = 0$), verifying the effectiveness of enforcing structural alignment between the learned representations and the constructed taxonomy.
Performance generally improves as $\lambda$ increases, peaking around $\lambda = 1.0$ on most datasets, and then gradually plateaus or slightly decreases.
This demonstrates the importance of balancing between task-specific supervision and structure-aware regularization.
Overall, TIRAG remains stable and effective across a broad range of $\lambda$ values, suggesting the robustness of our regularization strategy.

\begin{figure}[t]
\centering
\subfigure[Cora]{\includegraphics[width=0.49\linewidth]{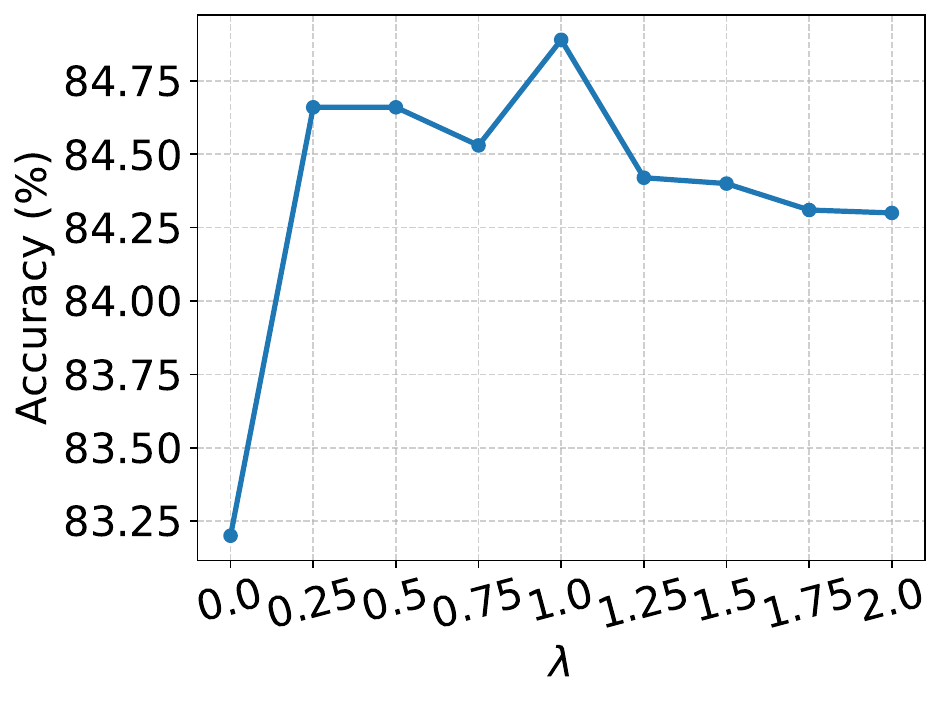}} 
\subfigure[Citeseer]{\includegraphics[width=0.49\linewidth]{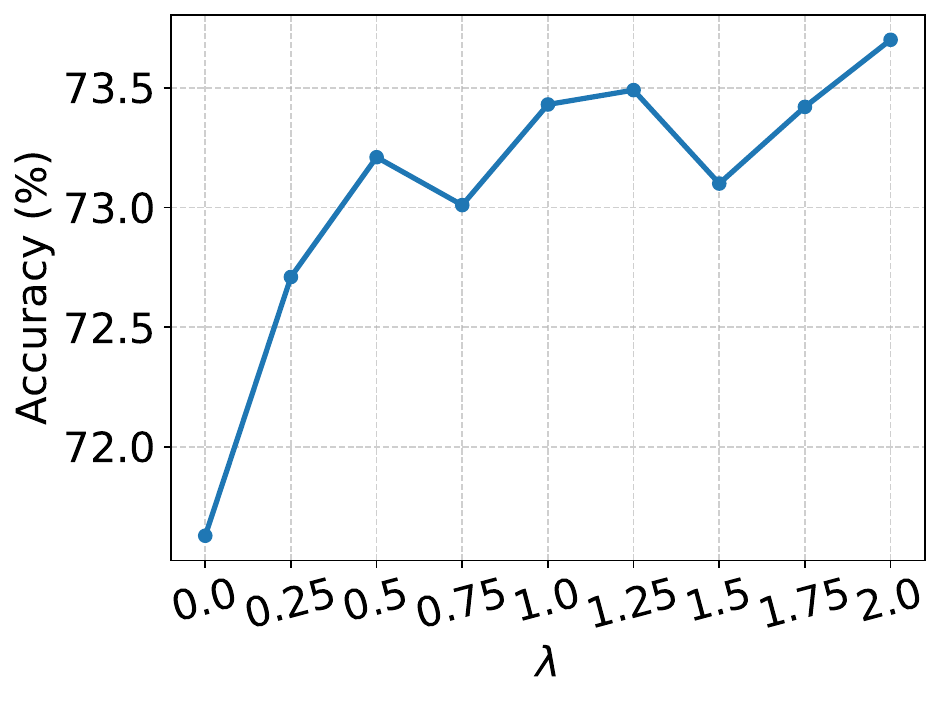}}
\subfigure[WikiCS]{\includegraphics[width=0.49\linewidth]{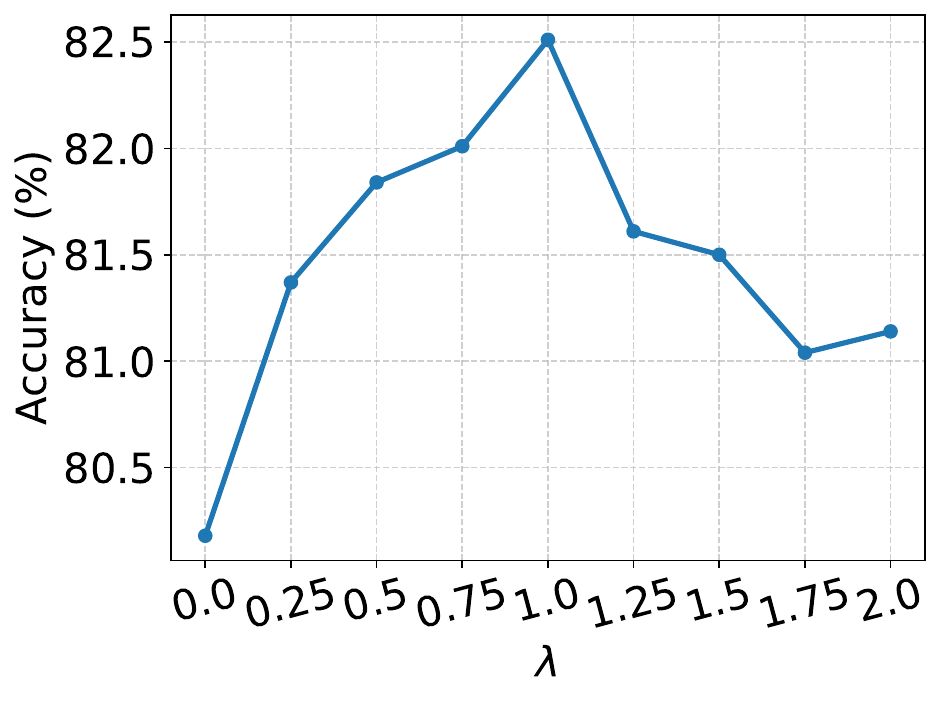}} 
\subfigure[Photo]{\includegraphics[width=0.49\linewidth]
{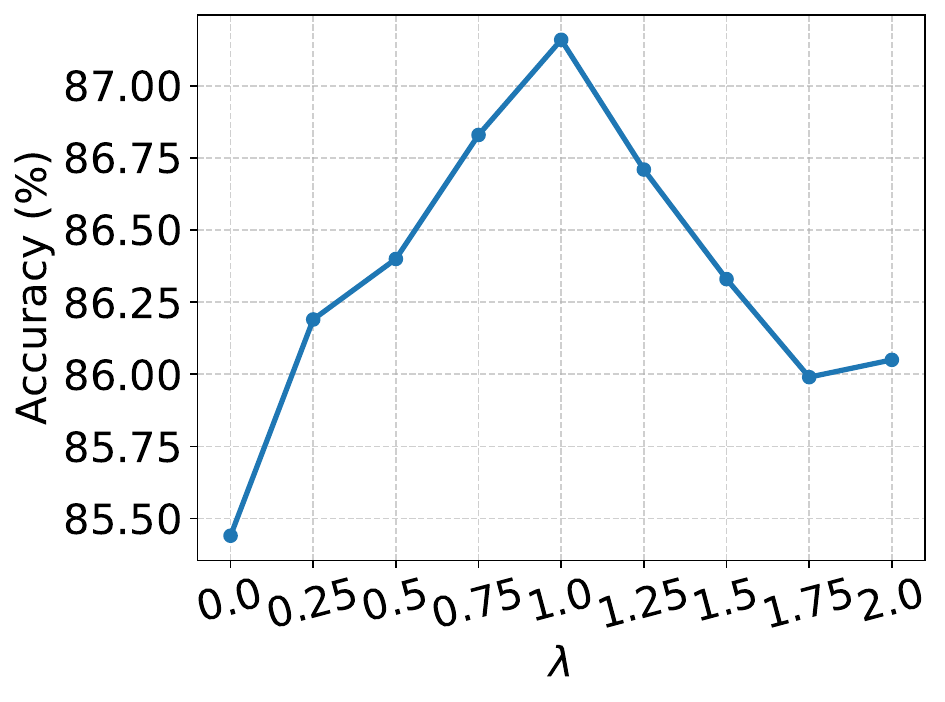}}
\caption{How the accuracy varies with different values of $\lambda$.}
\label{fig:lambda}
\end{figure}

\vpara{On the taxonomy tree shape.}
We further investigate the impact of the constructed taxonomy's shape, i.e., its depth and width, on the final performance.
Specifically, we vary the number of clusters per level and the number of hierarchical levels in the taxonomy tree, and evaluate the resulting classification accuracy on four representative datasets.
The results are shown in the tables below.
We observe that properly aligning the taxonomy shape with the characteristics of each dataset significantly improves performance. 
In particular, configurations with a number of top-level clusters approximately matching the number of ground-truth classes tend to yield better results. 
This suggests that aligning coarse-level taxonomy with label-level semantics provides a stronger structural prior, helping guide downstream node classification.

\subsection{Generalization to Link Prediction}

While our primary evaluation focuses on node classification, we further investigate whether the hierarchical representations learned by TIER generalize to other graph tasks, specifically link prediction. 
We note that downstream tasks more semantically related to taxonomy structures (e.g., classification, retrieval, and link prediction) are more likely to benefit from hierarchical knowledge, while tasks that are less taxonomy-aligned (e.g., anomaly detection) may observe limited gains.

\vpara{Setup.}
Following the experimental protocol of~\citet{LPBench}, we adopt three backbone models of increasing sophistication: GCN, Graph Autoencoder (GAE), and the state-of-the-art model BUDDY~\cite{BUDDY}. 
For each backbone, we compare the standard variant against one augmented with our taxonomy-informed regularization (``+Ours''). 
We report Mean Reciprocal Rank (MRR) on Cora and Citeseer.

\vpara{Results.}
Table~\ref{tab:link_prediction} summarizes the MRR results on Cora and Citeseer. 
Across all three backbone models and both datasets, incorporating our taxonomy-informed regularization consistently improves LP performance. 
For instance, on Cora, BUDDY benefits from a notable improvement of 3.76 points (26.40 → 30.16), and on Citeseer, GCN improves by 3.87 points (50.01 → 53.88). 
These gains can be attributed to the fact that hierarchical representations capture fine-grained semantic proximity between nodes, making semantically related nodes more likely to be correctly linked.
The results confirm that TIER enhances not only node classification but also other tasks that benefit from structured semantic organization.

\begin{table}[h]
\centering
\caption{Link prediction results (MRR, \%) on Cora and Citeseer. ``+Ours'' denotes the backbone augmented with our taxonomy-informed regularization.}
\begin{tabular}{l|cc|cc}
\toprule
\rowcolor{COLOR_MEAN} & \multicolumn{2}{c|}{Cora} & \multicolumn{2}{c}{Citeseer} \\
\rowcolor{COLOR_MEAN} Model & Base & +Ours & Base & +Ours \\
\midrule
GCN     & 32.50 & \textbf{34.69} & 50.01 & \textbf{53.88} \\
GAE     & 29.98 & \textbf{32.43} & 63.33 & \textbf{64.53} \\
BUDDY & 26.40 & \textbf{30.16} & 59.48 & \textbf{61.30} \\
\bottomrule
\end{tabular}
\label{tab:link_prediction}
\end{table}

\begin{figure*}[t]
\centering
\subfigure[Without Taxonomy]{\includegraphics[width=0.26\linewidth]{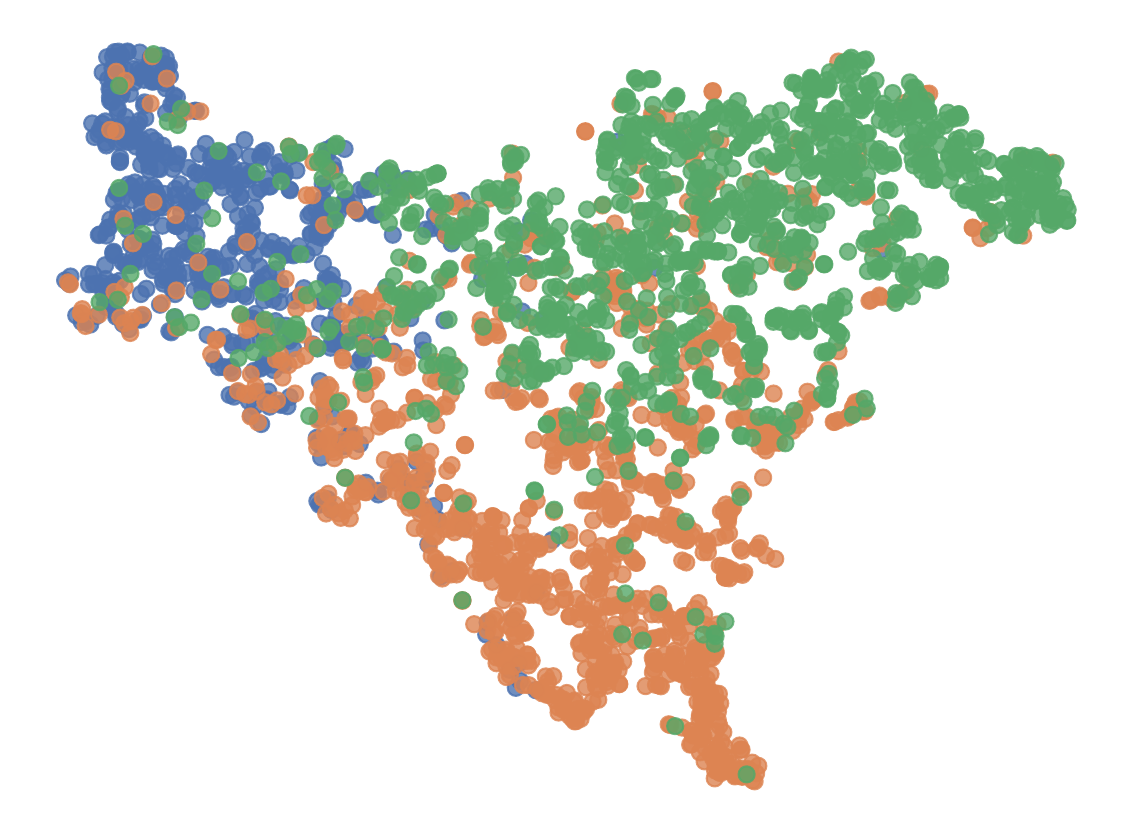}} 
\subfigure[Without Taxonomy]{\includegraphics[width=0.19\linewidth]{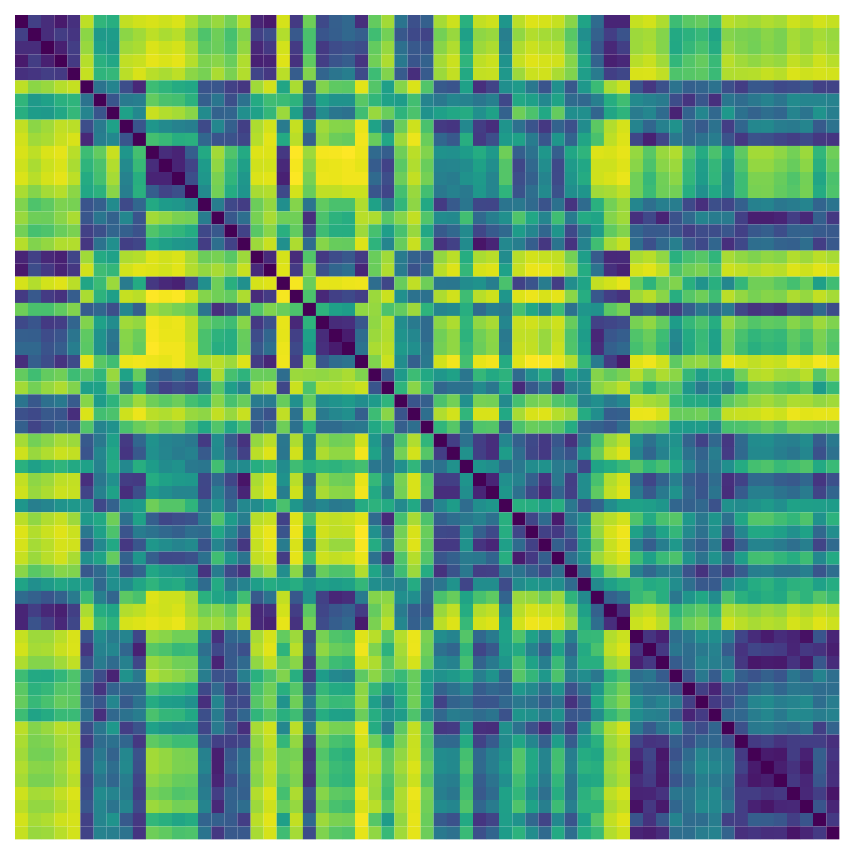}}
\subfigure[With Taxonomy]{\includegraphics[width=0.26\linewidth]{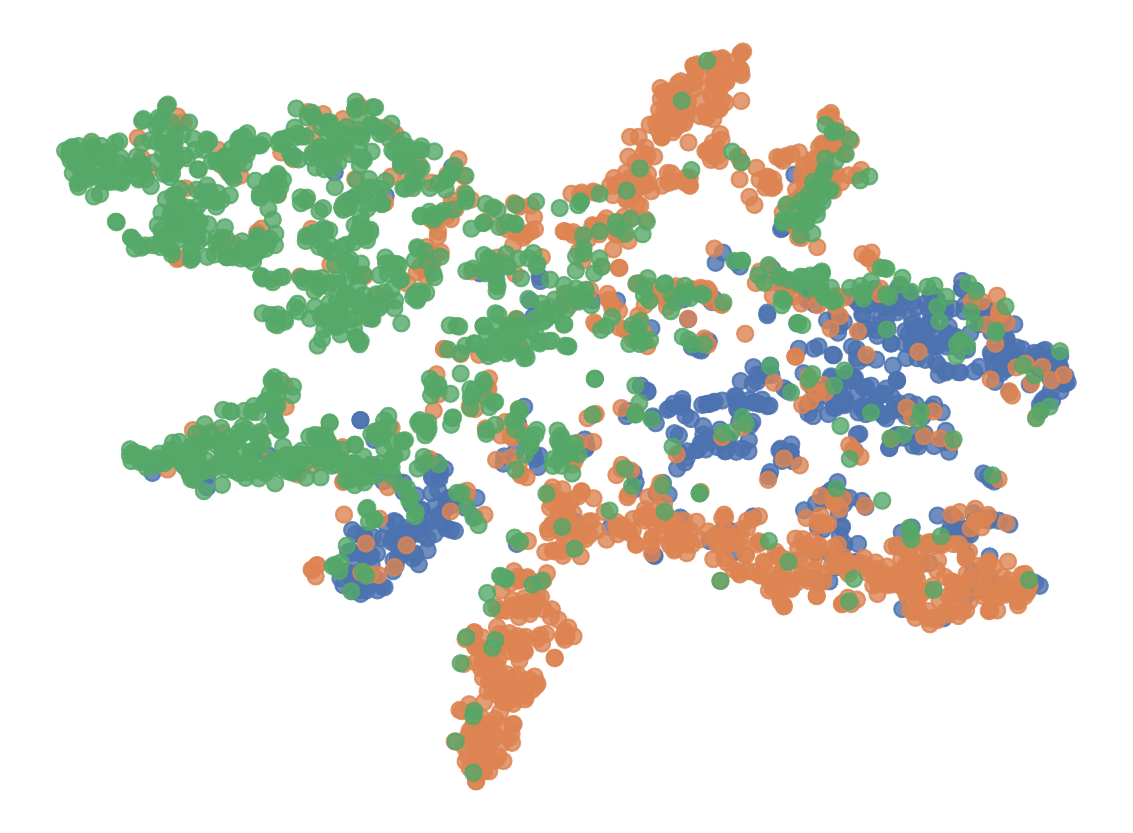}} 
\subfigure[With Taxonomy]{\includegraphics[width=0.19\linewidth]{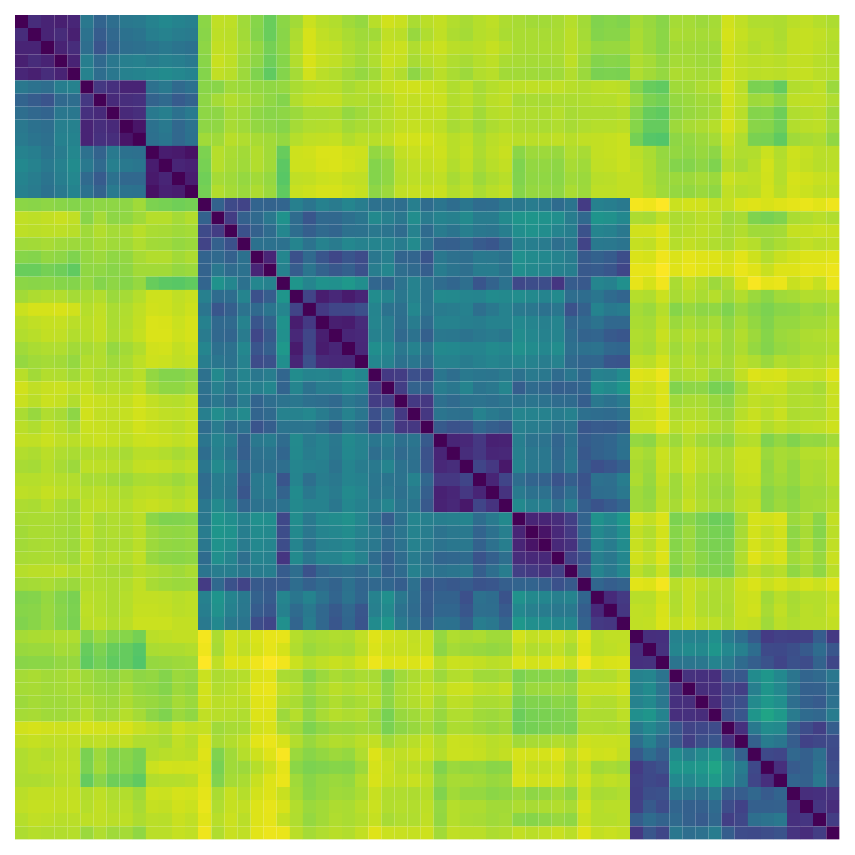}}
\caption{Visualizations of the learned node representations (colored by ground-truth labels) and the pairwise distance matrix between finest-level cluster centroids on Pumbed, with and without taxonomy regularization. The constructed taxonomy tree has 4 hierarchical levels with 1, 3, 16, and 63 cluster nodes from top to bottom, respectively. Darker colors indicate smaller distances. With taxonomy regularization, clearer block structures emerge, where darker diagonal blocks correspond to coarser-grained clusters, reflecting improved semantic hierarchy in learned representation space.}
\label{fig:visual_pubmed}
\end{figure*}

\begin{figure*}[t]
\centering
\subfigure[Without Taxonomy]{\includegraphics[width=0.26\linewidth]{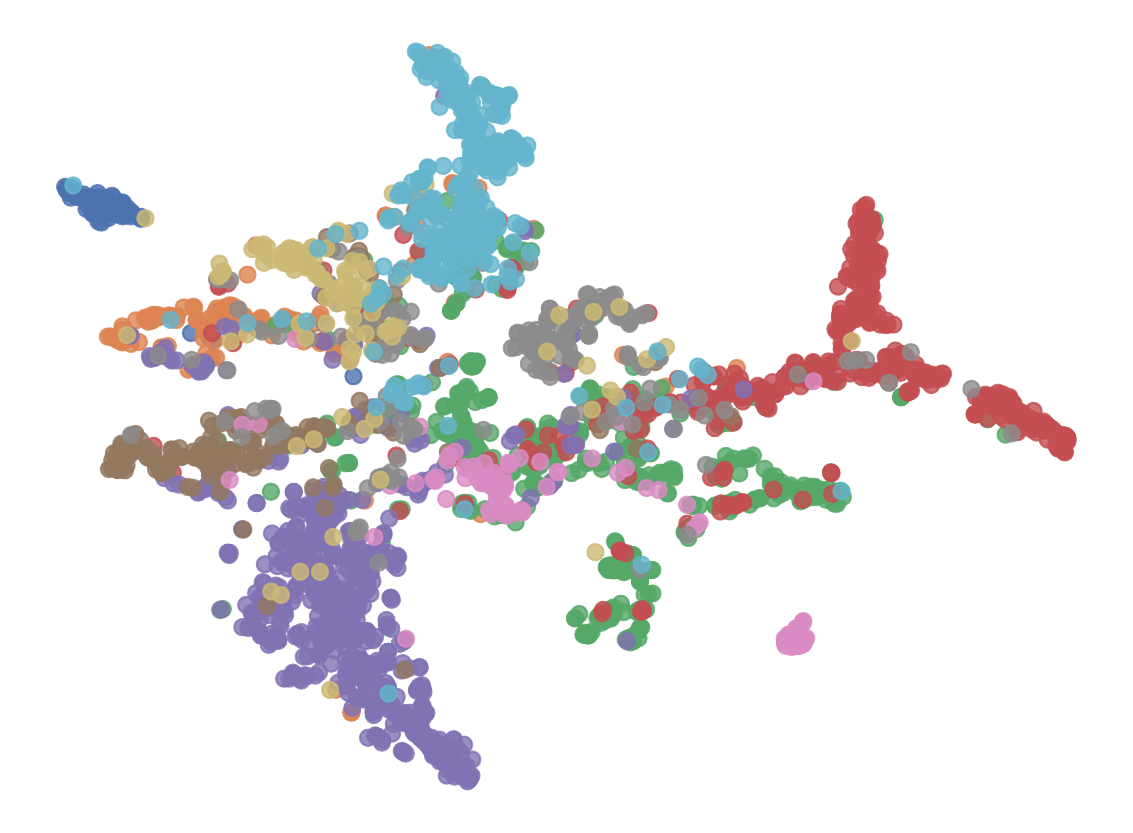}} 
\subfigure[Without Taxonomy]{\includegraphics[width=0.19\linewidth]{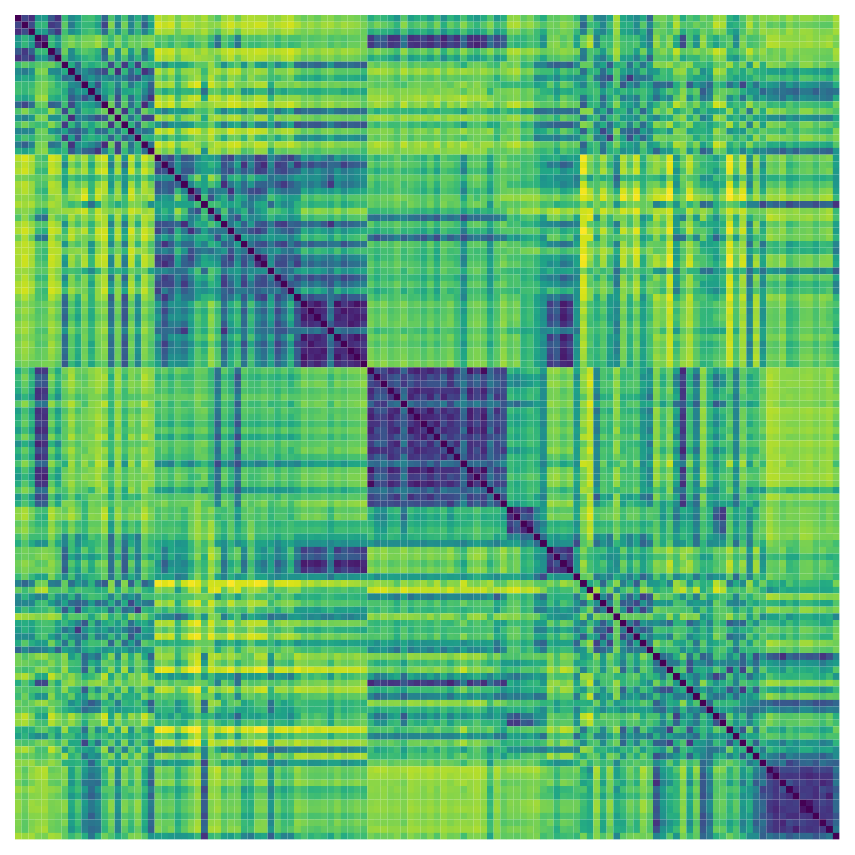}}
\subfigure[With Taxonomy]{\includegraphics[width=0.26\linewidth]{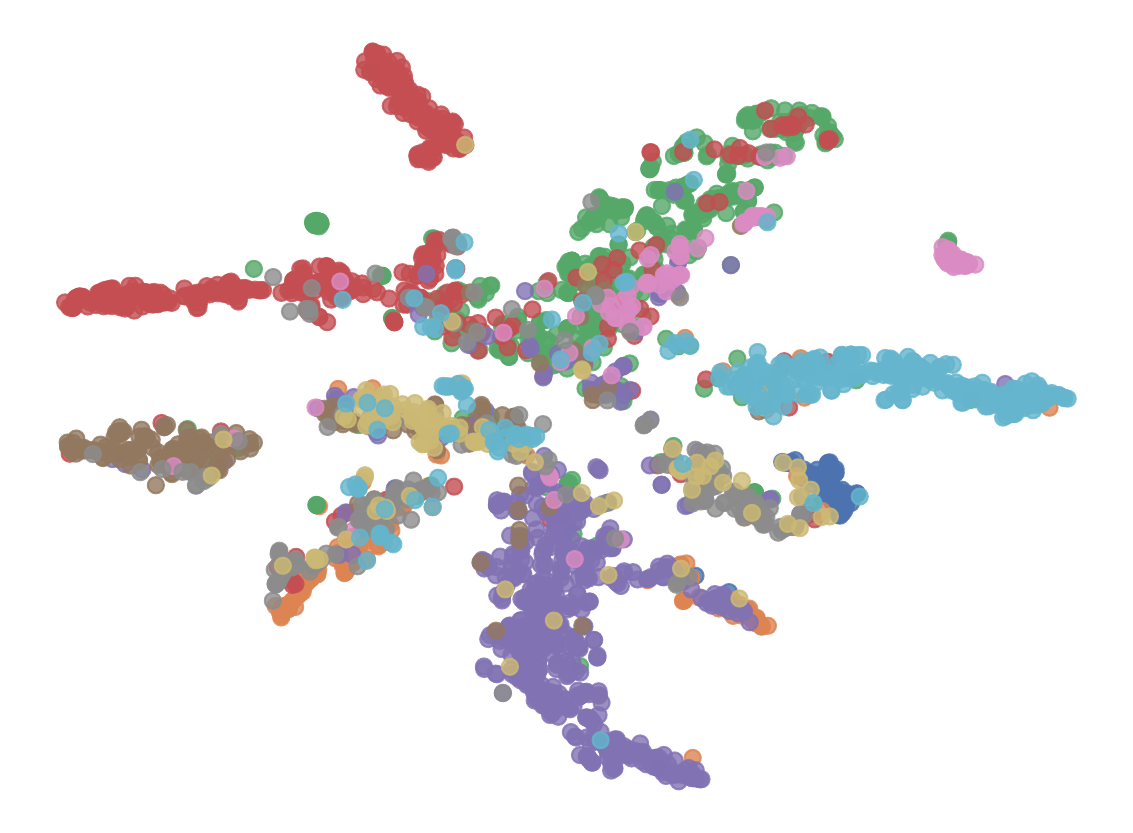}} 
\subfigure[With Taxonomy]{\includegraphics[width=0.19\linewidth]{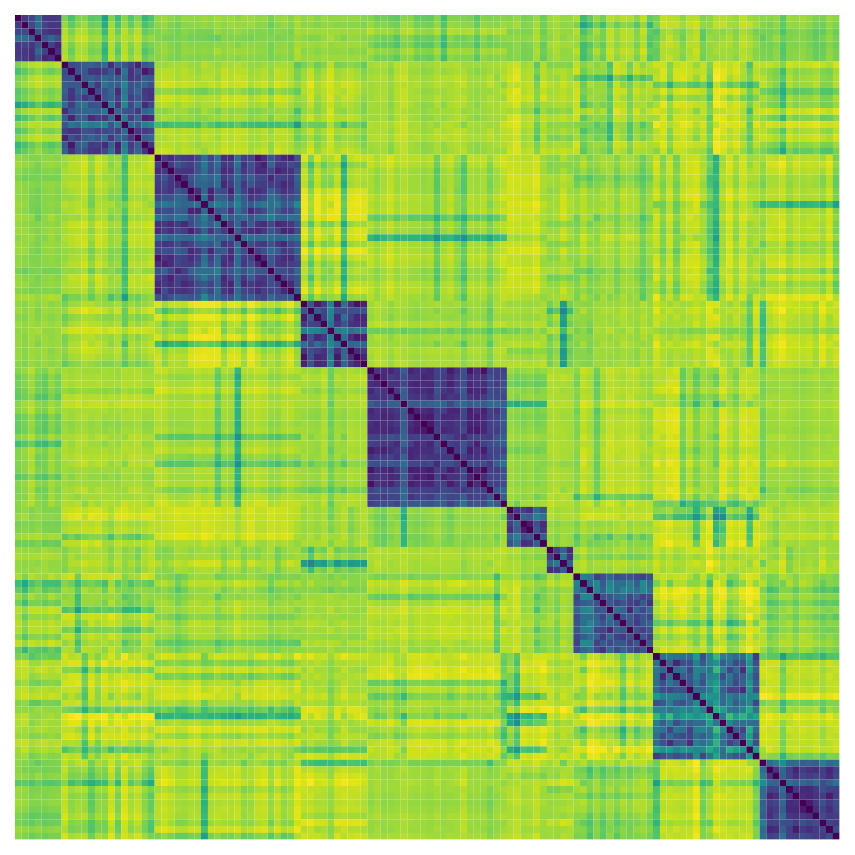}}
\caption{Visualizations of the learned node representations (colored by ground-truth labels) and the pairwise distance matrix between finest-level cluster centroids on WikiCS, with and without taxonomy regularization. The constructed taxonomy tree has 3 hierarchical levels with 1, 10, and 124 cluster nodes from top to bottom, respectively. Darker colors indicate smaller distances. With taxonomy regularization, clearer block structures emerge, where darker diagonal blocks correspond to coarser-grained clusters, reflecting improved semantic hierarchy in learned representation space.}
\label{fig:visual_wikics}
\end{figure*}

\begin{figure*}[t]
\centering
\subfigure[Without Taxonomy]{\includegraphics[width=0.26\linewidth]{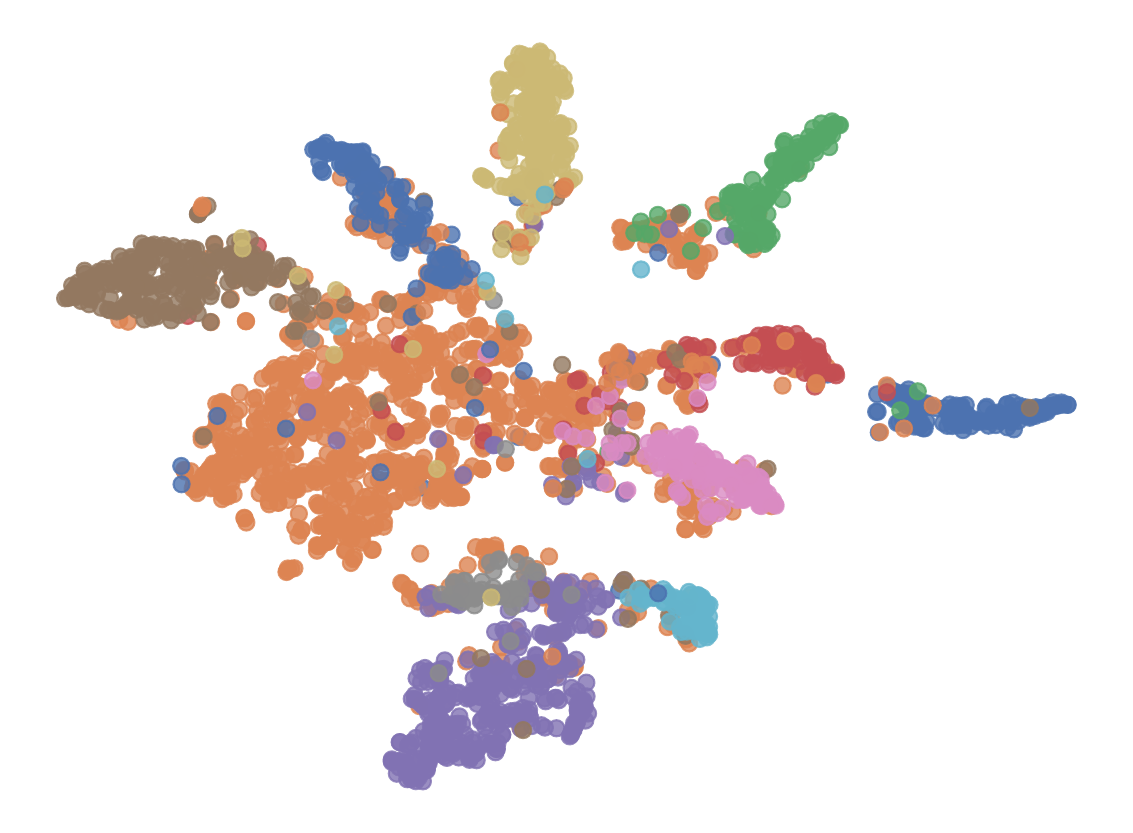}} 
\subfigure[Without Taxonomy]{\includegraphics[width=0.19\linewidth]{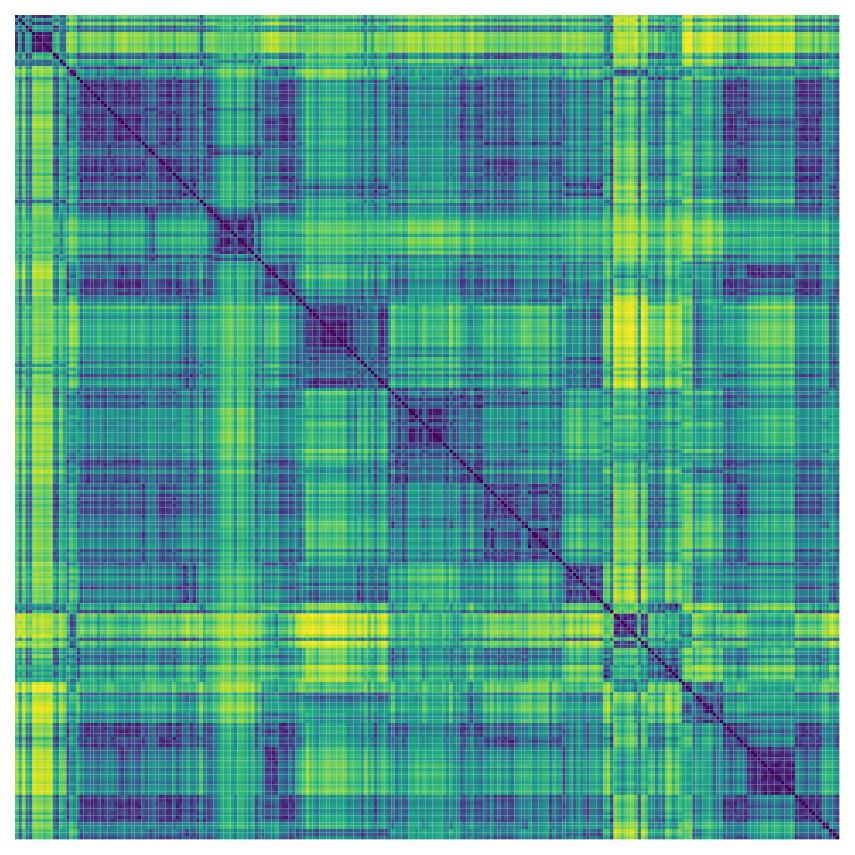}}
\subfigure[With Taxonomy]{\includegraphics[width=0.26\linewidth]{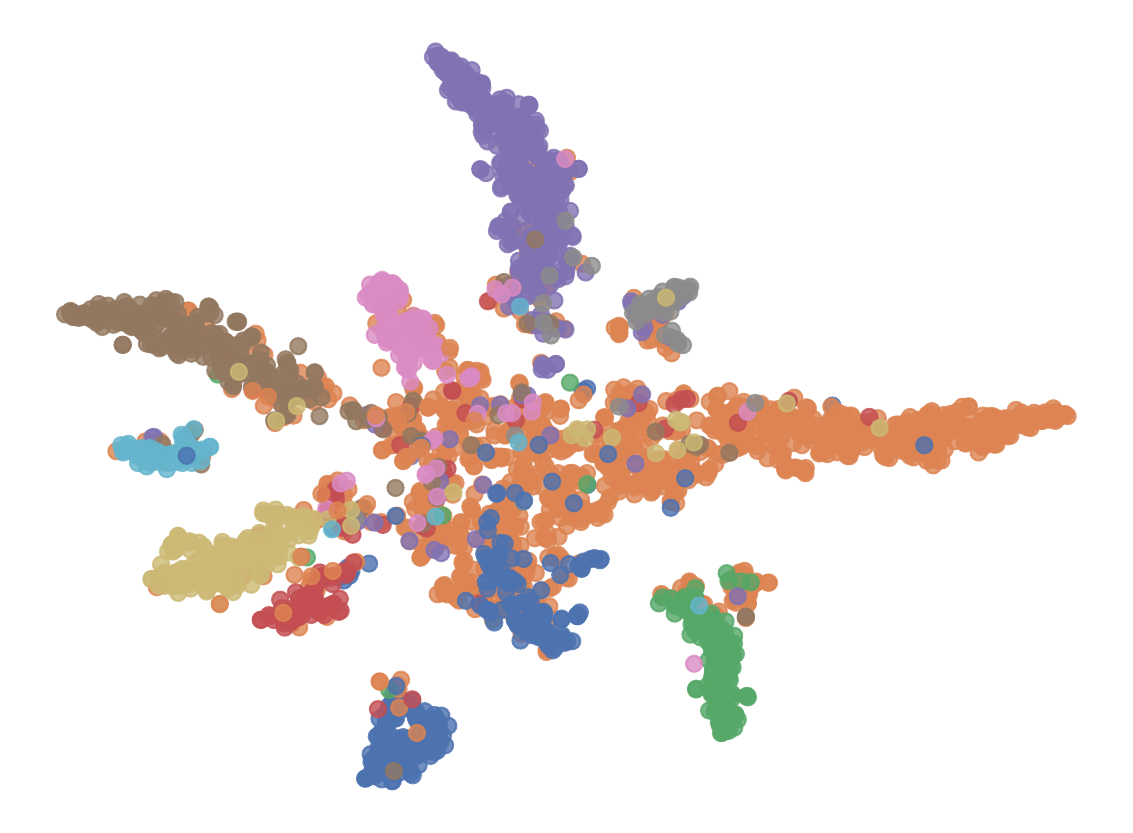}} 
\subfigure[With Taxonomy]{\includegraphics[width=0.19\linewidth]{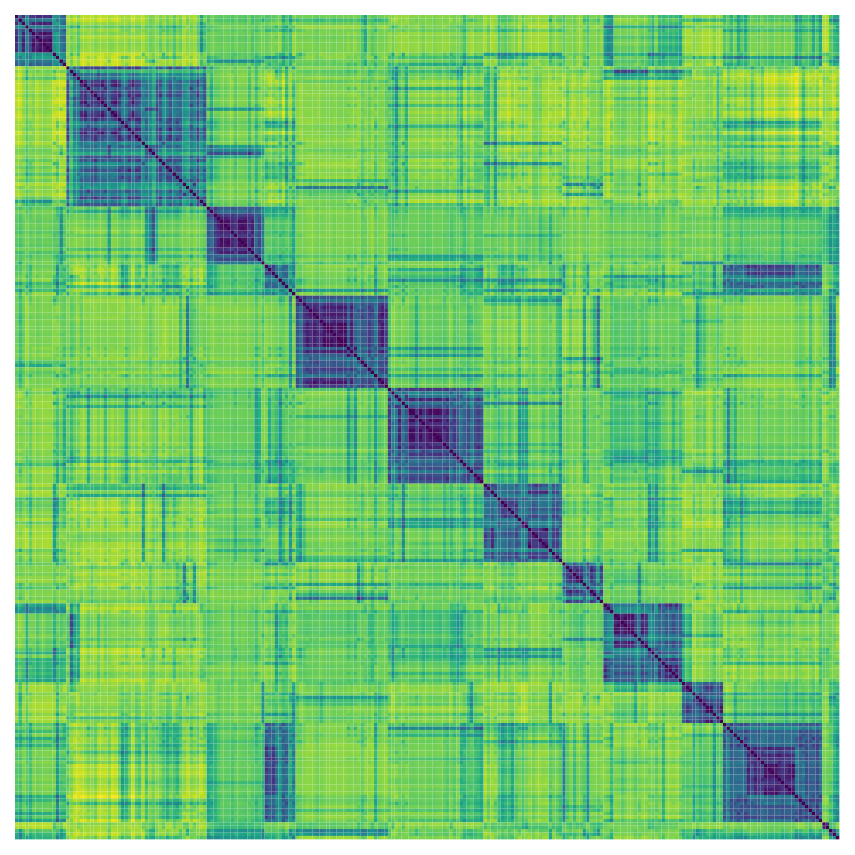}}
\caption{Visualizations of the learned node representations (colored by ground-truth labels) and the pairwise distance matrix between finest-level cluster centroids on Photo, with and without taxonomy regularization. The constructed taxonomy tree has 4 hierarchical levels with 1, 12, 64, and 241 cluster nodes from top to bottom, respectively. Darker colors indicate smaller distances. With taxonomy regularization, clearer block structures emerge, where darker diagonal blocks correspond to coarser-grained clusters, reflecting improved semantic hierarchy in learned representation space.}
\label{fig:visual_photo}
\end{figure*}


\begin{figure*}[h]
\centerline{\includegraphics[width=1.\linewidth]{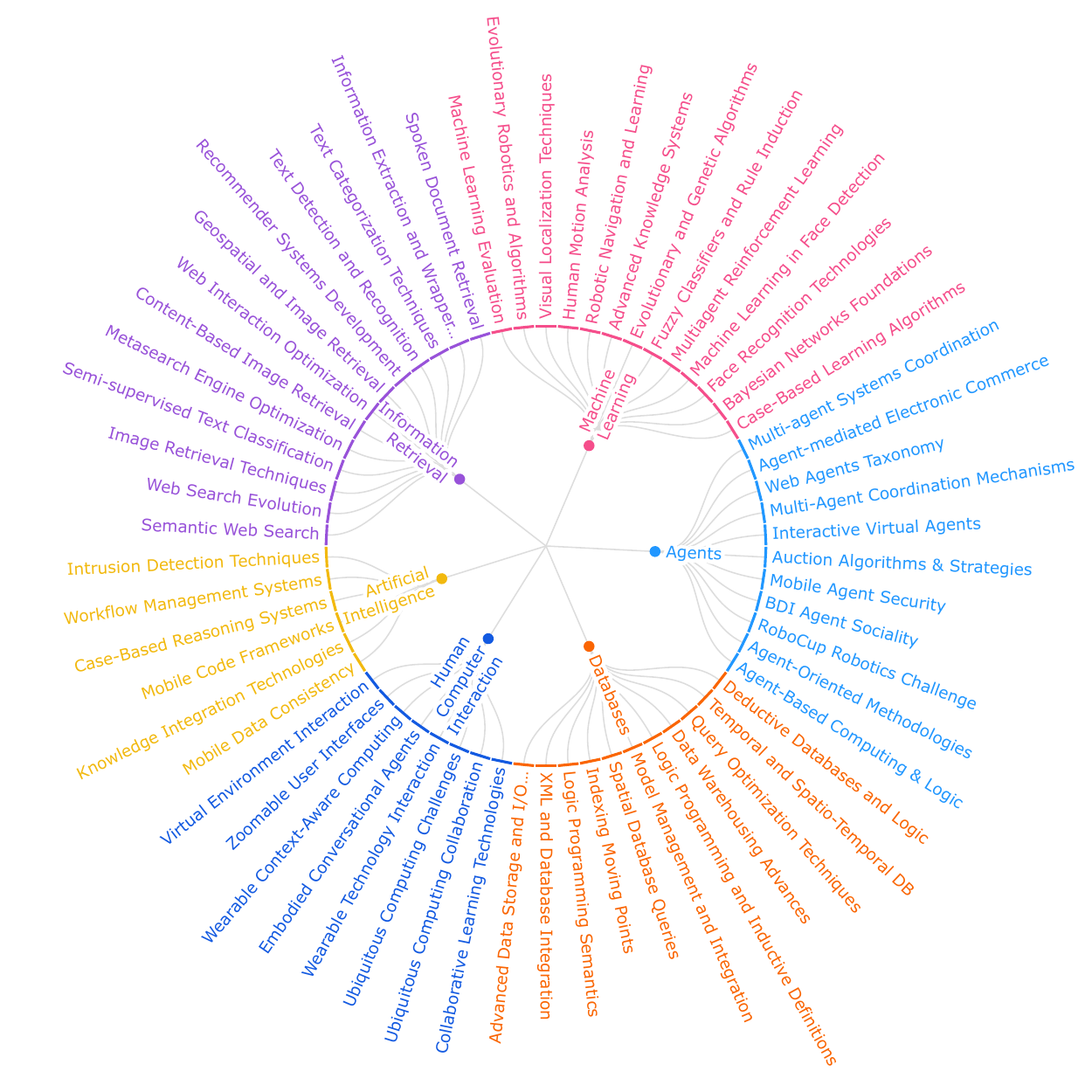}}\caption{RadialMap of the constructed taxonomy on Citeseer.}\label{fig:citeseer_taxonomy}
\end{figure*}

\begin{figure*}[h]
\centerline{\includegraphics[width=1.\linewidth]{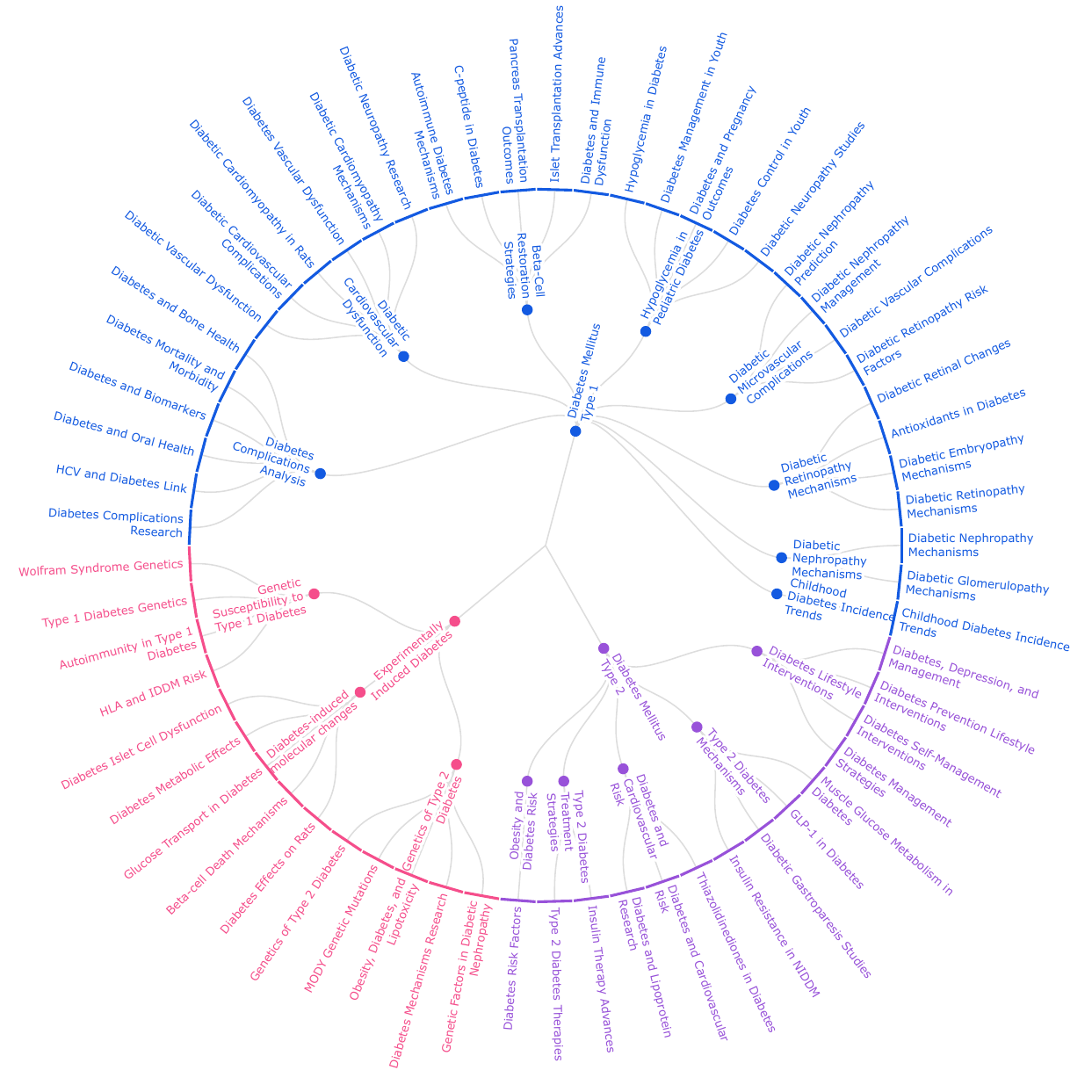}}\caption{RadialMap of the constructed taxonomy on Pubmed.}\label{fig:pubmed_taxonomy}
\end{figure*}

\begin{figure*}[h]
\centerline{\includegraphics[width=1.\linewidth]{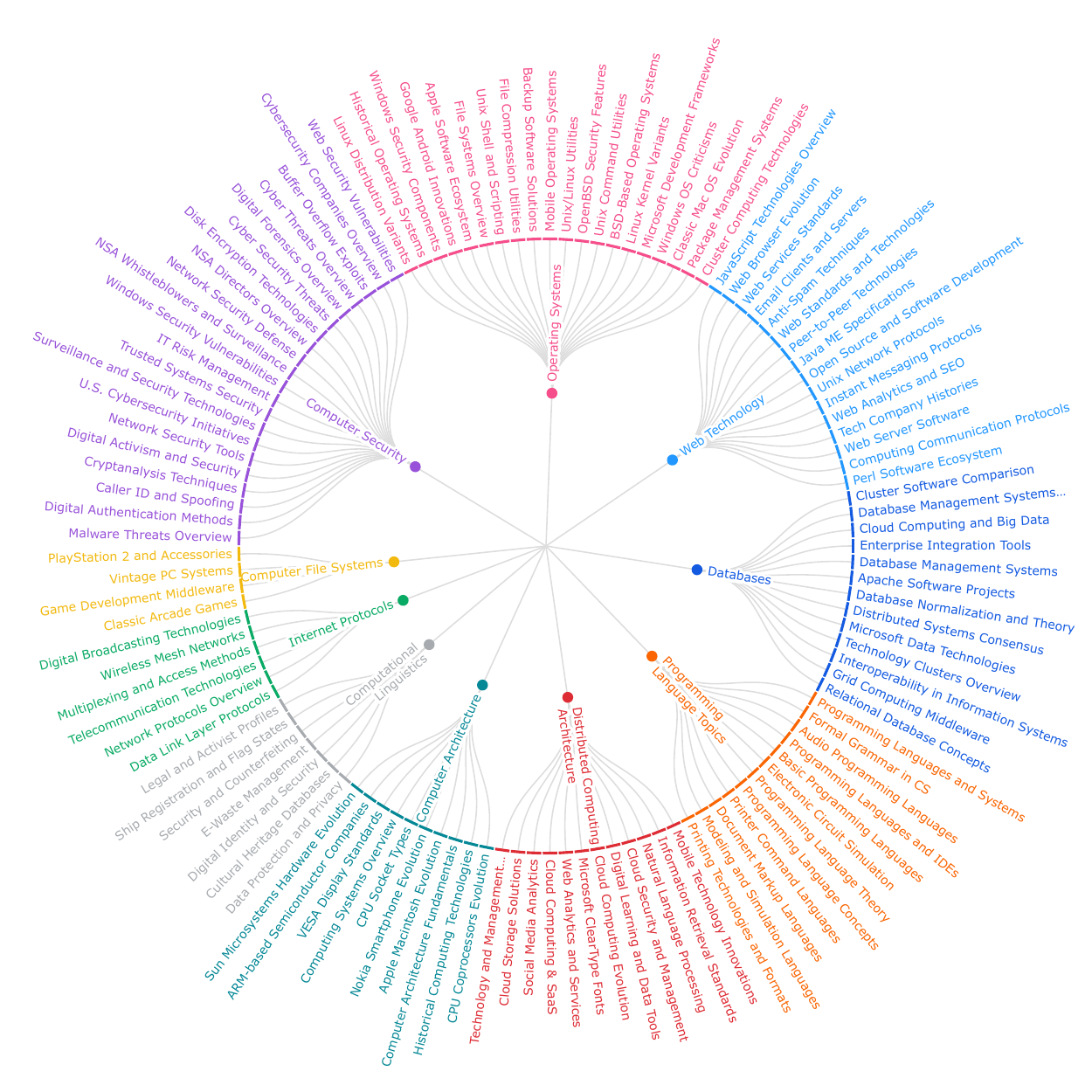}}\caption{RadialMap of the constructed taxonomy on WikiCS.}\label{fig:wikics_taxonomy}
\end{figure*}

\begin{figure*}[h]
\centerline{\includegraphics[width=1.\linewidth]{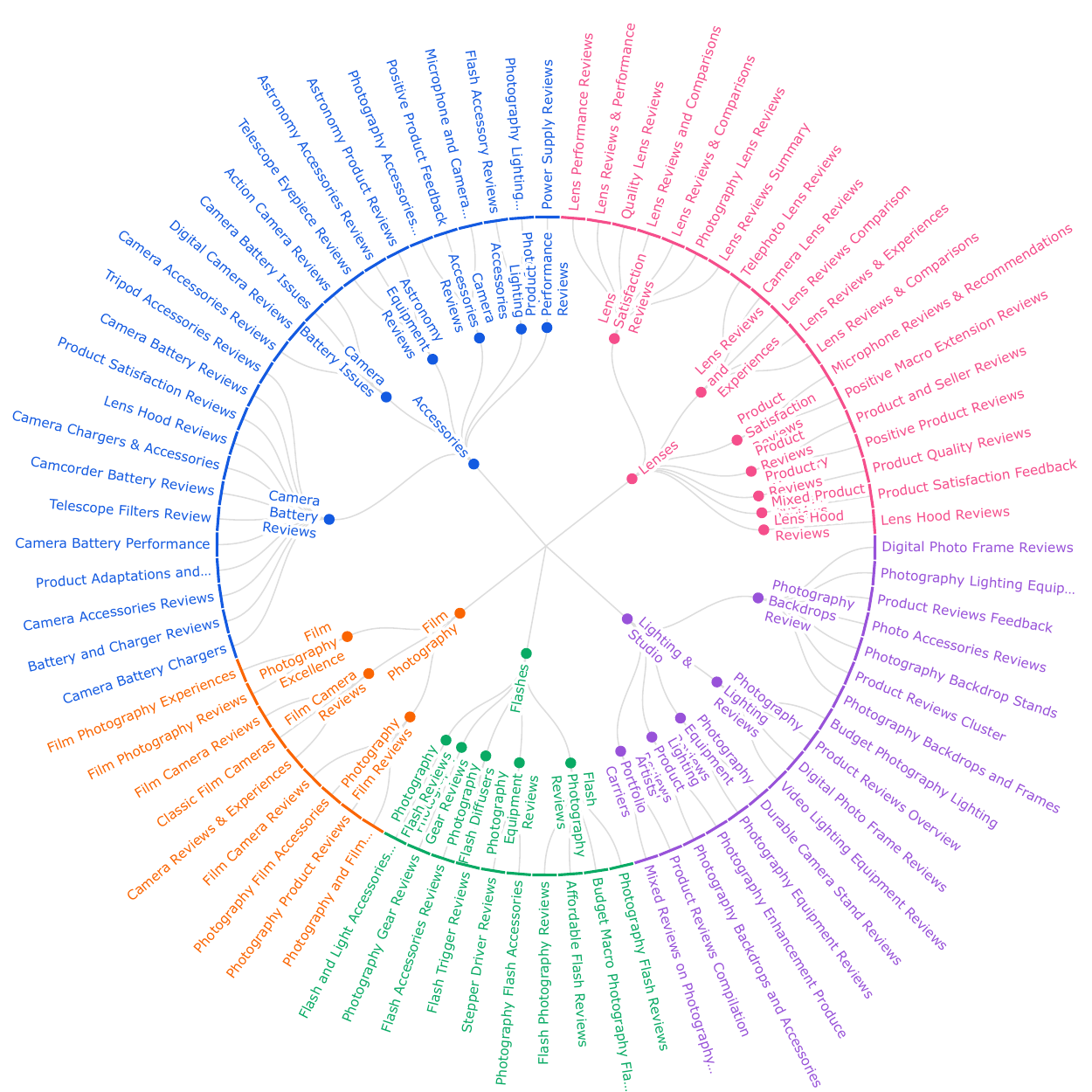}}\caption{RadialMap of the constructed taxonomy on Photo.}\label{fig:photo_taxonomy}
\end{figure*}
\end{document}